\pdfoutput=1

\documentclass[11pt]{article}

\usepackage[final]{acl}

\usepackage{times}
\usepackage{latexsym}

\usepackage[T1]{fontenc}

\usepackage[utf8]{inputenc}

\usepackage{microtype}

\usepackage{inconsolata}

\usepackage{graphicx}

\usepackage[utf8]{inputenc} %
\usepackage[T1]{fontenc}    %
\usepackage{hyperref}       %
\usepackage{url}            %
\usepackage{booktabs}       %
\usepackage{amsfonts}       %
\usepackage{nicefrac}       %
\usepackage{microtype}      %
\usepackage{xcolor}         %
\usepackage{multirow}
\usepackage{paralist}
\usepackage{graphicx}
\usepackage{setspace}
\usepackage{subcaption}
\usepackage{amsmath}
\usepackage{makecell}
\usepackage{algorithm}
\usepackage{algorithmic}
\usepackage{amsthm}

\usepackage{array}
\newcolumntype{L}[1]{>{\raggedright\let\newline\\\arraybackslash\hspace{0pt}}m{#1}}
\newcolumntype{C}[1]{>{\centering\let\newline\\\arraybackslash\hspace{0pt}}m{#1}}
\newcolumntype{R}[1]{>{\raggedleft\let\newline\\\arraybackslash\hspace{0pt}}m{#1}}
\newcolumntype{Z}[1]{>{\let\newline\\\arraybackslash\hspace{0pt}}m{#1}}
\newcolumntype{P}[1]{>{\centering\arraybackslash}p{#1}}

\DeclareMathOperator*{\argmin}{argmin}
\newtheorem{definition}{Definition}

\newtheorem{proposition}{Proposition}

\title{Robust AI-Generated Text Detection by Restricted Embeddings}

\author{
  \textbf{Kristian Kuznetsov\textsuperscript{1,4}},
  \textbf{Eduard Tulchinskii\textsuperscript{1,4}},
  \textbf{Laida Kushnareva\textsuperscript{1}},
  \\
  \textbf{German Magai\textsuperscript{2,3}},
  \textbf{Serguei Barannikov\textsuperscript{4,5}},
  \textbf{Sergey Nikolenko\textsuperscript{6,7}},
  \textbf{Irina Piontkovskaya\textsuperscript{1}},
\\
  \textsuperscript{1}AI Foundation and Algorithm Lab, Russia;
  \\
  \textsuperscript{2}HSE University, Russia;
  \textsuperscript{3}Noeon Research, Japan;
  \\
  \textsuperscript{4}Skolkovo Institute of Science and Technology, Russia;
  \textsuperscript{5}CNRS, Université Paris Cité, France;
  \\
  \textsuperscript{6}ISP RAS Research Center for Trusted Artificial Intelligence, Moscow, Russia;\\
  \textsuperscript{7}St. Petersburg Department of the Steklov Institute of Mathematics, Russia
\\
}

\begin{document}

\maketitle

\begin{abstract}
Growing amount and quality of AI-generated texts makes detecting such content more difficult. In most real-world scenarios, the domain (style and topic) of generated data and the generator model are not known in advance. In this work, we focus on the robustness of classifier-based detectors of AI-generated text, namely their ability to transfer to unseen generators or semantic domains. We investigate the geometry of the embedding space of Transformer-based text encoders and show that clearing out harmful linear subspaces helps to train a robust classifier, ignoring domain-specific spurious features. We investigate several subspace decomposition and feature selection strategies and achieve significant improvements over state of the art methods in cross-domain and cross-generator transfer. Our best approaches for head-wise and coordinate-based subspace removal increase the mean out-of-distribution (OOD) classification score by up to 9\% and 14\% in particular setups for RoBERTa and BERT embeddings respectively. We release our code and data\footnote{https://github.com/SilverSolver/RobustATD}.
\end{abstract}

\section{Introduction}

The proliferation of generative AI leads to an explosion in AI-generated content. Large language models (LLMs) can produce text that is very similar to human-written.
However, AI-generated text can be used for malicious purposes, which leads to the \emph{artificial text detection} (ATD) problem: has a given text or image been created by an AI model or a human?
Existing approaches for artificial text detection can be divided into \emph{score-based} and \emph{classifier-based}. The former aim to identify and measure features that distinguish artificial text from real; e.g., generated text may exhibit statistical artifacts due to the specific generation process used by a language model \cite{gehrmann2019gltr}, the difference may lie in perplexities measured by another language model \cite{solaiman2019release}, curvature of the probability function \cite{mitchell2023detectgpt}, or intrinsic dimensionality of contextualized representations \cite{tulchinskii2023intrinsic}. However, score-based methods often rely on prior knowledge about a specific generator and/or semantic domain, and known traces may be easy to remove, e.g., by paraphrasing the text \cite{krishna2023paraphrasing}. One notable exception is the intrinsic dimension feature for text content, shown to be robust to domain transfer and paraphrasing \cite{tulchinskii2023intrinsic}, but its overall detection quality is relatively modest. 

Supervised classification methods %
show almost perfect in-domain detection quality, but fail to generalize to unseen text topics and writing styles \cite{wang-etal-2024-m4, tulchinskii2023intrinsic}. The choice of training data, both artificial and generated, is crucial for successful out-of-distribution (OOD) transfer. In general, while usually there exist features that can distinguish between natural and artificial subsets of the training set, the classifier may lock into dataset-specific spurious differences and hence generalize poorly. It is hard to say in advance if a classifier trained on a given dataset will generalize well to new unseen generators and data sources. Previous approaches to OOD detection for ATD include UID-based detectors \cite{venkatraman2023gpt} and domain adversarial training \cite{bhattacharjee2024eagle}, but most of these methods are very data-intensive \cite{wang2024semeval}.

In this work, we aim to improve supervised classification by ignoring spurious features to enhance cross-distribution robustness, training on small number of domains or generator models. Namely, we focus on methods %
of extracting \textit{residual subspaces} and deleting information from embeddings.

In many applications, retaining only important dimensions of high-dimensional data while treating projections onto less loaded subspaces as residual noise can benefit downstream tasks. However, for tasks such as OOD detection the principal components of a dataset may be the least useful. \citet{kamoi2020mahalanobis} 
found that nullifying the first (least important) principal components in the embedding space fine-tuned on in-distribution (ID) data enhances OOD detection quality; this is known as the partial Mahalanobis distance. \citet{podolskiy2021revisiting} conducted similar analysis for Transformer-based text classifiers and found that ID data has orthogonal classes and lies on a unit sphere in a low-dimensional space. The main difference between ID and OOD data lies in the residual subspace, hence the partial Mahalanobis distance performs well in OOD detection.

It is important to note that not all neural networks learn useful residual subspaces for a given dataset; e.g., \citet{podolskiy2021revisiting} and \citet{ren2021simple} find that on text, CNN classifiers learn representations where components with low singular values contain too much information about ID data, making it difficult to distinguish OOD examples.

In this work, we apply similar techniques to artificial text detection (ATD). Distribution shift, with variations in text styles, topics, and new generation models, is a major challenge for ATD. Supervised classifiers, even performing well on validation datasets, struggle in realistic settings, where the domain and model of the AI-written text are unknown. To address this, we first show that training a classifier on some residual subspace obtained by coordinate removal or layer pruning may significantly enhance ATD robustness under domain and model shift.
Next, we show that controllable subspace removal can improve robustness, while also providing us with interpretable information about AI-written texts and domains. In particular, we use recent advances in concept erasure \cite{belrose2023leace}, experimenting with erasing semantic and syntactic concepts based on probing tasks by \citet{conneau-etal-2018-cram}; we show that some concepts are harmful for cross-domain and cross-model transfer.%

Our primary contributions are: (i) a first application of the residual subspace approach for robust ATD; we show that restricting the detector to a residual subspace increases cross-topic and cross-model robustness with especially significant improvements on the most difficult samples; (ii) analysis of different \emph{residual decomposition} techniques, such as nullifying head-wise subspaces in intermediate data representations and concept erasure; (iii) analysis of applicability of our methods with different encoder- or decoder-based backbone models. Besides, we create and release an extension for one of the datasets with recent generating model GPT-4-o on three domains. 
Below, Section~\ref{sec:related} surveys related work,  Section~\ref{sec:methods} describes the proposed methods, Section~\ref{sec:data} introduces the datasets, Section~\ref{sec:eval} presents a comprehensive experimental evaluation, and Section~\ref{sec:concl} concludes the paper.

\section{Related Work}\label{sec:related}

\textbf{Linear subspaces in Transformer-based models} are known to represent concepts.
\citet{hernandez-andreas-2021-low} studied low-dimensional subspaces that encode linguistic features in BERT; linear structure is known for such concepts as truthfulness \cite{marks2023geometry} and sentiment \cite{DBLP:journals/corr/abs-2310-15154}. This direction has been extended to the \textit{linear representation hypothesis} that posits that language models operate with one-dimensional representations of concepts in the activation space \cite{bricken2023towards, park2023the}. However, \citet{engels2024not} showed that some concepts are multi-dimensional.

\textbf{Components of Transformer-based embeddings}
can provide useful features via the geometry of their inner representations or parameter spaces.
For instance, \emph{outlier dimensions} in the embedding spaces of models such as BERT \cite{DBLP:conf/naacl/DevlinCLT19} or RoBERTa \cite{liu2019roberta}, characterized by unusually high variance and/or mean values, have been studied in detail, including
their emergence during training and effects of disabling them post-training \cite{kovaleva-etal-2021-bert}, their relationship with positional embeddings and impact on word-in-context tasks \cite{luo-etal-2021-positional}, influence on the quality of representations \cite{timkey-van-schijndel-2021-bark}, and relations to the shapes of attention maps and token frequencies \cite{puccetti-etal-2022-outlier}.
Activations of Transformer-based LMs have been investigated for language structure information \cite{jawahar-etal-2019-bert}, semantic and syntactic features \cite{conneau-etal-2018-cram}; the latter work also introduces a comprehensive selection of probing tasks. 
However, 
only outlier dimensions have been studied in full detail;
we aim to address this gap by studying how removing specific dimensions from RoBERTa embeddings can improve detection of artificially generated text.

\textbf{Semantics of attention heads in Transformers}
have been studied for a long time: \citet{Kovaleva2019RevealingTD} provided empirical research on BERT attention heads, demonstrating overparameterization by pruning some of them, \citet{Michel2019AreSH} showed that most heads can be removed at test time without significant performance loss. \citet{clark-etal-2019-bert} studied the specialization of attention heads; \citet{pande2021heads}, their functional roles.
In BERT-like models, important information is distributed across layers; e.g., \citet{jawahar-etal-2019-bert} showed that lower layers capture phrase-level information, which is partly lost in the upper layers; the model captures a hierarchy of linguistic information, with deeper layers required to capture long-distance dependencies. \citet{bian-etal-2021-attention} showed that attention maps are correlated across layers and organized into clusters. Therefore, here we focus on groups of attention heads within a layer rather than individual heads.

\textbf{Artificial text detection} (ATD)
is a new field of study (up until recently, artificial content was mostly easy to distinguish), but there already exist many promising approaches. 
Score-based methods include DetectGPT, which measures the curvature of the probability function \cite{mitchell2023detectgpt}, and GPTZero \cite{tian2023gptzero}, which checks the perplexity and burstiness of a text; these methods, however, are limited to a single domain or generator. %
Throughout recent work, it remains a reasonable baseline for general and cross-distribution ATD to take embeddings from BERT-like models as a feature space and train logistic regression (LR) over them. Following~\citet{tulchinskii2023intrinsic} and~\citet{Jawahar20}, we take the RoBERTa model \citep{liu2019roberta} to extract text embeddings, use mean-pooling over embeddings, and train LR models for ATD.
The recent SemEval-2024 competition ~\cite{wang2024semeval} proposed challenge in a multi-generator, multi-domain, and multi-language setting based on the new ATD dataset that was introduced in \citet{wang-etal-2024-m4}. Task 8 included problems such as binary classification, source identification, and fake/real text boundary detection. Solutions used approaches such as LLM fine-tuning (RoBERTa, XLM-R), contrastive learning, and ensemble methods.
However, while all these approaches are data-intensive, absolute classification quality is still poor. In this work, we use classifiers that perform well on in-domain data and aim to improve their performance on unseen domains. %

\section{Methods}\label{sec:methods}

Removing unnecessary features is often an effective method to improve the robustness of a machine learning model. The embedding space has linear substructures responsible for linguistic features such as token frequencies, word-in-context information etc. \citep{luo-etal-2021-positional, puccetti-etal-2022-outlier}. We aim to detect and erase such substructures, which are harmful for ATD generalization.

\def\z{\mathbf{z}}
\def\x{\mathbf{x}}
\def\bc{\mathbf{c}}
\def\cC{\mathcal{C}}

\subsection{Linear decompositions of embeddings}

\textbf{PCA and the standard basis}. Let $\x$ be some text input, $\z\in \mathbb{R}^d$, its embeddings obtained by some model, $\z=M(\x)$, $\cC = \{\bc_1,\ldots, \bc_d\}$, a basis of $\mathbb{R}^d$, and let $\alpha_i$ be coefficients of $\z$ in $\cC$, $\z = \sum_{i=1}^d\alpha_i\bc_i$.
We want to split $\cC$ into \emph{good} and \emph{bad} parts, $\cC=\cC_{g}\cup \cC_{b}$, so that
components in $\cC_g$ contain most of the information general for all domains, while $\cC_b$ is responsible for spurious domain-specific features. Then, we construct a classifier on \emph{restricted} embeddings $\z'$ where the ``bad'' part is nullified, $\z' = \sum_{i\in \cC_g}\alpha_i\bc_i$.
Intuitively, information about the style, topic, and other semantic properties is harmful for ATD, and we want to focus on residual features that are less important for other NLP tasks.
\citet{podolskiy2021revisiting} show that PCA can serve as such a decomposition for a Transformer-based model:
removing top components computed for an in-domain dataset improves OOD detection. Indeed, for a dataset of natural texts $\mathcal{D}$, subspace $\langle\cC_b\rangle$ should ``explain'' the data variability, while the variance of $\mathcal{D}$ projected on $\langle \cC_g\rangle$ is expected to be low. PCA is a theoretically optimal way to find such subspaces (see Appendix~\ref{sec:appendix-def-and-theory}).

\def\be{\mathbf{e}}

Despite PCA's solid theoretical background, in practice it does not always perform well; in ATD, we usually deal with a small dataset that cannot fully capture the real distribution, which is bad for PCA.
To access data properties beyond those represented in our train set, we propose to utilize the internal structure of the pretrained embedding model. Indeed, Transformer-based models tend to \emph{disentangle} some data properties during training, and semantic interpretation has been discovered for some neurons and embedding dimensions \cite{luo-etal-2021-positional, timkey-van-schijndel-2021-bark, puccetti-etal-2022-outlier}. We hypothesize that such ``built-in'' disentanglement could lead to meaningful subspaces spanned by a subset of the \emph{standard basis}, i.e., vectors $\{\be_1=[1,0,\dots,0], \be_2=[0,1,\dots,0], \dots, \be_d=[0, 0,\dots, 1]\}$. Projection to a subspace $\langle \be_i | i\in S\rangle$ for some subset of indices $S$ can be done by simply nullifying all embedding dimensions except $S$. Our experiments support this intuition: PCA-based decomposition does not lead to any significant changes in detector's quality (see Appendix~\ref{sec:appendix_pca}), while coordinate-based subspace removal significantly improves transfer scores.

\textbf{Attention heads as linear substructures}. Both decompositions discover \textit{global} linear structure, i.e., universal directions in the embedding space independent of input data. But it is much more natural to rely on \emph{local} linearity of the data and try to discover substructures in a data manifold that does not necessarily form global linear subspaces. 
For text embeddings, the neural network represents a function from $\mathbb{R}^{d\times T}$ to a data manifold $\mathcal{M}$. We can decompose this function into a sum of input-dependent components of the same functional form.
\citet{cammarata2020thread} proposed \emph{linear circuits}, showing that the data flow in a Transformer can be represented as the main residual stream with linear addition of flows from other elements of the model  (attention heads and feed-forward blocks). We are mostly interested in attention flows because it is well known that attention heads in Transformers have highly specialized functions \cite{Kovaleva2019RevealingTD, pande2021heads}, so we hypothesize that head-wise decomposition should reflect the ``built-in'' disentanglement of the pretrained model.
We can represent a Transformer-based embedding as
\begin{equation}
    \label{eq:attn_decomposition}
    \resizebox{.89\linewidth}{!}{$\z = \Pi\left[\alpha(\x) \x_0 + \sum\limits_l \beta_l(\x) \mathrm{MLP}^l + \sum\limits_l \sum\limits_h \gamma_l(\x) A^{l,h}\right],$}
\end{equation}
where $ A^{l,h}$ are the outputs of attention heads, $\alpha$, $\beta$, $\gamma$ are scalar functions, and $\Pi$ is a centering projection $\Pi(\x) = \x-\frac{1}{d}\sum\nolimits_{i=1}^{d}\x^{i}$ (see Appendix~\ref{sec:appendix-headwise-decomposition}).

\textbf{Concept erasure}. Finally, we consider an embedding space decomposition based on extracted linear directions or low-rank subspaces responsible for some harmful semantic feature $\z_F$. If such a direction is found, we can remove it by subtracting the component corresponding to this direction from the embedding. Namely, we erase the feature as
\begin{equation}
\label{eq:feature_projection}
\hat{\z} = \z - P_F (\z),  
\end{equation}
where $P_F$ is the projection to the subspace $\z_F$.

\subsection{Subspace removing methods} 
\label{sec:removing}

\def\dsearch{D_{\mathrm{search}}}
\def\deval{D_{\mathrm{eval}}}

\textbf{Greedy search}. Our basic approach chooses the best features using a small subset of domains.
Given a multi-domain dataset $D = D_1\cup \dots \cup D_k$, where $D_i$ are domain subsets, we choose two domains $\dsearch = \{D_1, D_2\}$ to perform feature selection.
On each step, we train a classifier on $D_1$, removing one component, and look how its performance changes on $D_2$, getting a feature ranking on $D_1 \rightarrow D_2$ transfer. Then, we do the opposite, getting a ranking for $D_2 \rightarrow D_1$. The final set of residual features is obtained as the union of top-score lists in both rankings (see Appendix~\ref{sec:appendix_gready}).

\textbf{Head pruning} removes some components in decomposition \eqref{eq:attn_decomposition} by replacing the output of a given head with zeros on inference. Importantly, this approach is approximate because, besides its direct impact as a component in the decomposition, each head also has an indirect influence on all computations on subsequent layers. But \citet{gandelsman2023interpreting} showed that this indirect impact is small and can be ignored (see also Appendix~\ref{sec:appendix-headwise-decomposition}).
To choose the set of heads for pruning, we note that different layers contain different kinds of information
(e.g., semantic information is mostly in bottom and middle layers), and the linguistic complexity of tasks solved by attention heads grows from bottom to top \cite{Kovaleva2019RevealingTD, tenney-etal-2019-bert}. Therefore, 
we simply prune every layer separately.

\textbf{Concept erasure by probing tasks}. To remove a linear subspace responsible for some data properties, we apply a concept erasure technique called LEACE \cite{belrose2023leace}. Suppose we have a  $k$-class classification task defined by a dataset $Z$ with one-hot labels $Y$, and we want to erase all the knowledge required for linear separation of the classes. LEACE is a projection-based method of the form \eqref{eq:feature_projection}, with theoretical guarantees that any linear classifier on top of $\hat{\z}$ cannot solve the classification task better then a constant predictor. Erasing a concept from an embedding $\z$ is defined as
\begin{equation}
    \hat{\z} = \z - W^{+} (W \Sigma_{ZY}) (W \Sigma_{ZY})^{+} W \z,
\end{equation}
where $W = (\Sigma_{ZZ}^{1/2})^{+}$, $\Sigma_{ZZ}$ is $Z$'s covariance matrix, $\Sigma_{ZY}$ is the cross-covariance of $Z$ and $Y$. Geometrically, LEACE is the least-squares-optimal transform that maps centroids of different classes of the dataset $(Z, Y)$ to the same point, making linear separation impossible. 

In this work, we utilize probing tasks provided by \citet{conneau-etal-2018-cram}, designed to represent elementary linguistic concepts (see Section~\ref{sec:data}). These experiments allow us to not only improve ATD robustness, but also obtain insights about the influence of interpretable linguistic features.

\section{Data}\label{sec:data}

\textbf{ATD datasets}.
There are few high-quality datasets with both human and artificial text. One such dataset was presented by \citet{wang-etal-2024-m4} and used in the SemEval-2024 competition\footnote{\url{https://semeval.github.io/SemEval2024/}}; it covers five domains: Wikipedia, Reddit, WikiHow, PeerRead, and arXiv. We have used five text generation models: GPT-3.5 \cite{chatgptpost}, Davinci003\footnote{\url{https://platform.openai.com/docs/models}}, Cohere\footnote{\url{https://docs.cohere.com/docs/models}} , Dolly-v2 \cite{dollytechreport}, and BLOOMz \cite{muennighoff2023crosslingual}. Since the amount of human-written text in each domain is larger than generated by each model, we crop human data so that there are 3000 samples of parallel data for each domain and model/human combination. 

\begin{figure}[!t]
   \centering
   \includegraphics[width=\linewidth]{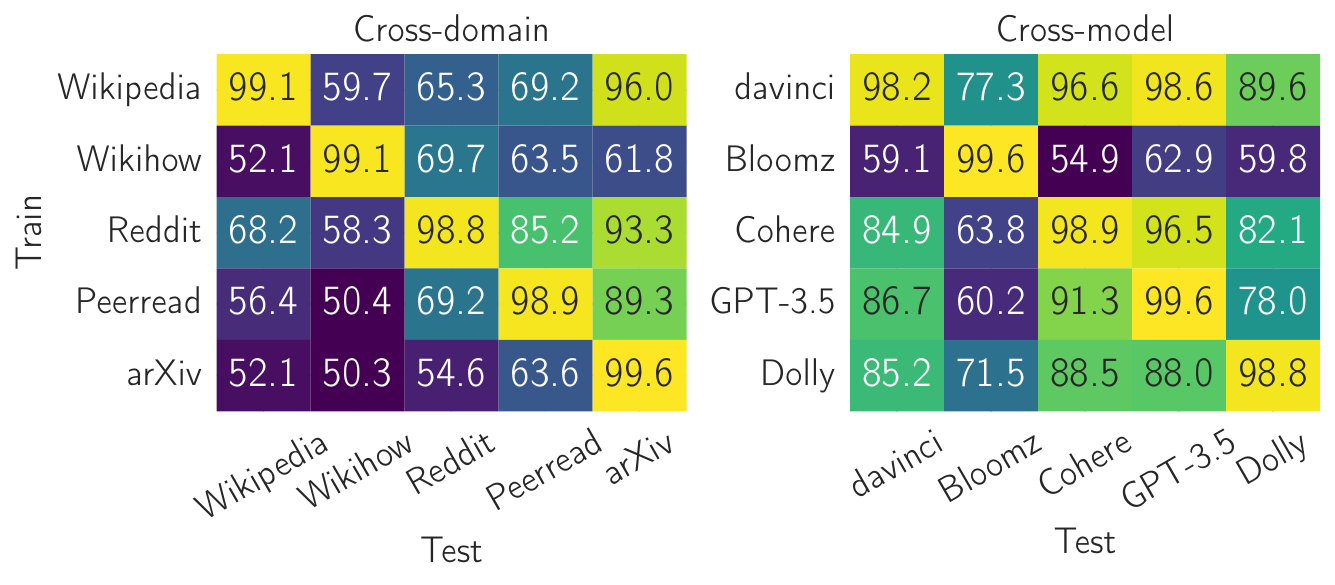}
   \caption{Mean accuracy in cross-domain (left) and cross-model ATD by RoBERTa-base on SemEval}
   \label{fig:roberta_semeval}
\end{figure}

Our second dataset, used by \citet{tulchinskii2023intrinsic}, has three domains---Wikipedia, Reddit, StackExchange---with davinci003 generations. Compared to SemEval, it has a larger distribution shift in basic text features (e.g., length), which makes it harder for cross-domain transfer. We extend it by adding similar text generated by GPT-4o: continuing text from Wikipedia articles, long-form question answering on Reddit Q\&A and StackExchange. Thus, we obtain a dataset, called below GPT-3D, with six domain-model pairs.

\textbf{Experimental setup}. 
Similar to \citet{wang-etal-2024-m4}, we create two tasks for SemEval dataset:
\begin{inparaenum}[(1)]
\item in the \emph{cross-domain} task, we concatenate data across generating models, getting five binary ATD tasks in different domains;
\item in the \emph{cross-model} task, concatenation across domains yields five binary ATD tasks for each generator model.
\end{inparaenum}
Thus, results are presented as 
$5 \times 5$ heatmaps (e.g., Fig.~\ref{fig:roberta_semeval}) and its aggregations.

 For GPT-3D
we report average OOD scores, i.e. the accuracy of classifiers trained on one domain-model subset and evaluated on the rest; average accuracy values do not include training sets.

For more technical details, see Appendix~\ref{app:technical}.

\textbf{Probing datasets}.
For probing and concept erasure experiments, we use the dataset used by \citet{conneau-etal-2018-cram} with
several supervised classification tasks:
\emph{SentLen}, predicting the length of the sentence, \emph{TreeDepth}, finding the depth of a syntactic tree, \emph{TopConst}, classifying the high-level syntactic structure (top two nodes in the syntax tree), classifying 
\emph{Tense}, \emph{SubjNum} (subject number), and \emph{ObjNum} (object number) in the main clause, 
detecting errors with \emph{BShift} (bigram shift, word order inversion in a bigram), \emph{SOMO} (Semantic Odd Man Out, where a word is replaced with a random grammatically fitting word), and
\emph{CoordInv} (Coordination Inversion,
whether the coordination of two clauses in the sentence is inverted),
and predicting exact words from a 1000-word vocabulary in \emph{WC} (Word Content).

\section{Results and Analysis}\label{sec:eval}

Here we present results on baseline, heads pruning, concept erasure and selecting coordinates. PCA-based results are reported in Appendix~\ref{sec:appendix_pca}.

\def\tblangle{50}

\begin{table}[!t]\centering\setlength{\tabcolsep}{2pt}\small
\begin{tabular}{lccccc}
\toprule
Domains   &  \rotatebox{\tblangle}{Wikipedia} & \rotatebox{\tblangle}{\textcolor{red}{WikiHow}} & \rotatebox{\tblangle}{Reddit} &    \rotatebox{\tblangle}{PeerRead}    &  \rotatebox{\tblangle}{\textcolor{red}{arXiv}}     \\
\midrule
Avg. transfer to:              & 57.2 & \textcolor{purple}{54.7} & 64.7 & 70.4 & \textcolor{teal}{85.1} \\
Avg. transfer from:              & 72.5 & 61.8 & \textcolor{teal}{76.3} &     66.3 & \textcolor{purple}{55.2} \\
Avg. sent. length & 38.7 & \textcolor{teal}{44.4} & 17.0 & 14.7 & \textcolor{purple}{10.4} \\
Avg. ``!'' count & 0.24 & \textcolor{teal}{0.79} & 0.25 & 0.08 & \textcolor{purple}{0.01} \\
Avg. ``?'' count & 0.12 & \textcolor{teal}{0.90} & 0.36 & 0.43 & \textcolor{purple}{0.03} \\\midrule
Generators   &  \rotatebox{\tblangle}{davinci} & \rotatebox{\tblangle}{\textcolor{red}{Bloomz}} & \rotatebox{\tblangle}{Cohere} &    \rotatebox{\tblangle}{GPT-3.5}    &  \rotatebox{\tblangle}{Dolly}     \\
\midrule
Avg. transfer to:              & 79.0 & \textcolor{purple}{68.2} & 82.8 & \textcolor{teal}{86.5} & 77.4 \\
Avg. transfer from:              & \textcolor{teal}{90.5} & \textcolor{purple}{59.2} & 81.8 & 79.1 & 83.3 \\
Avg. sent. length & 17.7 & \textcolor{purple}{10.9} & 15.3 & \textcolor{teal}{22.6} & 21.2 \\
Avg. ``!'' count & 0.18 & \textcolor{teal}{0.50} & \textcolor{purple}{0.04} & 0.23 & 0.29 \\
Avg. ``?'' count & \textcolor{purple}{0.08} & \textcolor{teal}{0.39} & 0.11 & 0.13 & 0.22 \\
\bottomrule
    \end{tabular}
\caption{Average RoBERTa detector accuracy by domains and by generators on \emph{SemEval} (in \%), avg length of generated sentences (in symbols) and avg counts of ``!'' and ``?'' marks per text sample; \textcolor{purple}{dark red}~-- smallest value in a row, \textcolor{teal}{dark green}~-- largest value, \textcolor{red}{red}~-- domains and generators with lowest transfer accuracy.}\label{tab:hard_to_transfer}
\end{table}

\textbf{Baseline RoBERTa}.
As a baseline we use logistic regression (LR) trained on mean-pooled RoBERTa embeddings. Results are shown in Fig.~\ref{fig:roberta_semeval} for SemEval and Fig.~\ref{fig:roberta_gpt4}a for GPT-3D; the cross-domain and cross-model settings are challenging in both tasks. Fig.~\ref{fig:roberta_semeval} shows that in-domain classification is almost perfect for baseline LR on RoBERTa embeddings, but the cross-domain part is very inconsistent: e.g., transfer from Reddit to PeerRead works well across all models ($91\%$ avg accuracy) but transfer from arXiv to WikiHow is uniformly bad ($54\%$).
In \textit{SemEval}, \textit{Wikihow} is the hardest domain to transfer to, while \textit{Arxiv} is the hardest domain to transfer from (Table~\ref{tab:hard_to_transfer}); both domains contain syntactic anomalies (very few or many ``!'' and ``?'' marks, unusual average sentence lengths etc.). \textit{Bloomz} is the hardest model to transfer both to and from (Table~\ref{tab:hard_to_transfer}), and it also generates unusual texts (very short sentences replete with ``!'' and ``?''). But generally, it is not easy to predict which transfer direction is easier in ATD or explain the reasons for it; e.g., \textit{Wikipedia}, often used for NLP model evaluation \cite{merity2016pointer}, is far from the best basis for transfer, especially in the cross-model setting (Fig.~\ref{fig:roberta_gpt4}a).
We also compare (Fig.~\ref{fig:roberta_gpt4}f) our proposed methods with the approach based on the intrinsic dimensionality (PHD) of real and artificial texts tokens embedding point clouds, according to \cite{tulchinskii2023intrinsic}.

\textbf{Average transfer results.} Table~\ref{tab:main_table} and Fig.~\ref{fig:roberta_gpt4} show that our methods provide a stable improvement of OOD scores for classifiers trained on separate  domain-model subsets, for both SemEval and GPT-3D datasets. \emph{TopConst} concept erasure yields the highest increase among methods that do not have access to OOD data ($+3\%$), and improvement increases for the most difficult domain pairs (e.g., $+6\%$ for \emph{Wikipedia}--\emph{Reddit}).
Interestingly, the PHD method by \citet{tulchinskii2023intrinsic}, while providing very stable cross-domain results for GPT-3-based generations, completely fails to deal with GPT-4o (Fig.~\ref{fig:roberta_gpt4}~(f)), while our methods increase cross-model scores up to $10\%$. Still, results for the most difficult pairs are unsatisfactory, falling below the random baseline; the only method that can achieve at least random level for \textit{any} OOD subset is head pruning, where the heads are selected on validation set combined of all models and domains examples ($+9.1\%$ ``cross-all'' compared to full RoBERTa, Fig.~\ref{fig:roberta_gpt4}~(e)). Further we describe results for each method in details and in the Appendix~\ref{app:combination} we describe the combination of methods.

\begin{table}[!t]
\centering \setlength{\tabcolsep}{2pt} \small
{
\begin{tabular}{lcccccc}
\toprule
        & \multicolumn{3}{c}{\textbf{Cross-domain}}              & \multicolumn{3}{c}{\textbf{Cross-model}}              \\
        & Avg           & Max $\uparrow$  & Max $\downarrow$  & Avg           & Max $\uparrow$ & Max $\downarrow$  \\
        \midrule
RoBERTa & 73.0          & -             & -             & 82.8          & -            & -             \\
1       & \textbf{76.0} & \textbf{18.9} & -7.1          & 82.6          & \textbf{4.4} & -4.4          \\
2       & 73.9          & 6.3           & -3.7          & 83.3          & 2.8          & -2.4          \\
3       & 75.0          & 8.6           & -1.9          & 83.1          & 2.6          & -1.8          \\
4       & 74.6          & 8.4           & -1.6          & \textbf{83.7} & 3.3          & -1.6          \\
5       & 73.7          & 3.6           & -1.8          & 82.9          & 1.7          & -1.4          \\
6       & 72.6          & 2.8           & -3.7          & 82.7          & 1.6          & -2.1          \\
7       & 72.3          & 1.3           & -5.2          & 82.5          & 1.2          & -3.0          \\
8       & 73.3          & 3.5           & -3.5          & 82.5          & 0.4          & -1.8          \\
9       & 73.1          & 4.5           & \textbf{-1.5} & 82.7          & 0.6          & -1.2          \\
10       & 72.7          & 3.2           & -3.2          & 82.4          & 0.4          & -2.2          \\
11      & 73.2          & 3.8           & -6.5          & 82.3          & 0.4          & -1.4          \\
12      & 73.7          & 7.2           & -3.7          & 82.8          & 1.7          & \textbf{-1.1}\\
\bottomrule
\end{tabular}
}
\caption{Balanced accuracy for OOD classification for different pruned layers on \emph{SemEval}}
\label{tab:text_results_layers}
\end{table}

\textbf{Head pruning for transfer tasks}.
We adapt head pruning \cite{voita-etal-2019-analyzing} to remove a whole layer of attention heads. Since layers of a model have rough linguistic meanings \cite{jawahar-etal-2019-bert}, thus we analyse the impact of structural-level information on ATD. 
Fig.~\ref{fig:roberta_semeval_pruning_detailed} and Table~\ref{tab:text_results_layers} show detailed results for each layer pruned on \emph{SemEval}.
Removing the first layer improves average cross-domain accuracy by $3\%$, but the improvement is unstable (from $-7.1\%$ to $+18.9\%$ in different domains). 
Pruning layers 3 and 4 is more stable and beneficial in both settings.
Cross-domain ATD is more challenging;
Fig.~\ref{fig:roberta_semeval_pruning_detailed} (top) shows that some domains (\textit{Wikipedia} and \textit{WikiHow}) exhibit similar patterns but others are unrelated. The best scores are in transfer from \emph{Reddit}, achieving $81\%$ mean balanced accuracy with $0$-th layer pruned ($+5\%$ to full RoBERTa).
The cross-model setting is easier and not greatly affected by pruning layers, with the exception of BLOOMz. Here
the best source model is GPT-3.5-davinci, with $92\%$ cross-model accuracy after removing layer 4.

\begin{figure}[!t]
   \centering
   \includegraphics[width=\linewidth]{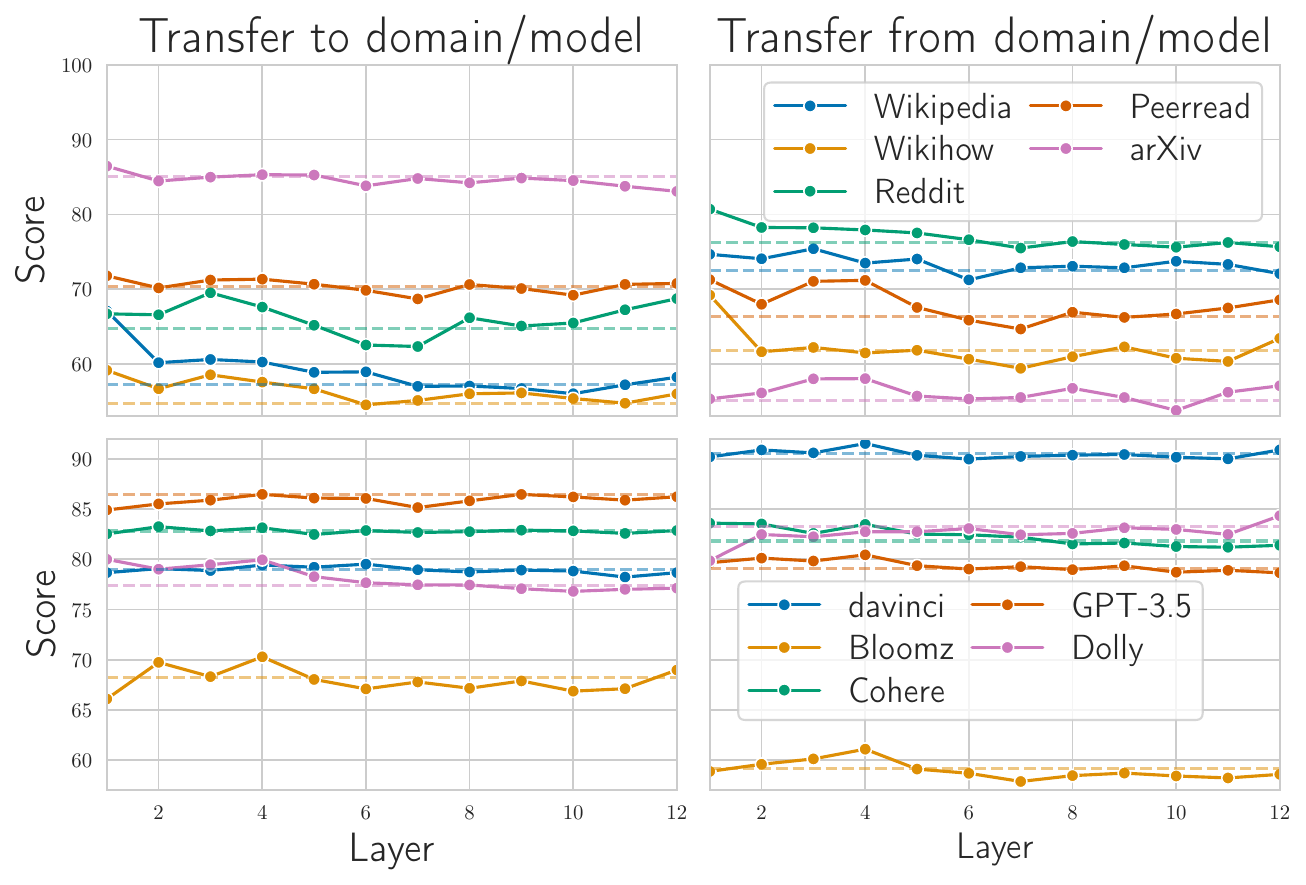}
   \caption{Mean accuracy on \emph{SemEval} with pruned RoBERTa layers. Dashed lines show the baseline.}
   \label{fig:roberta_semeval_pruning_detailed}
\end{figure}

\textbf{Concept erasure}.
Generally, results on \emph{SemEval} show that the best concepts to erase are \emph{TopConst} and \emph{TreeDepth}, improving up to 2.1\% on cross-domain transfer and not hurting the cross-model transfer. Erasing \emph{WC} also performs well but is less stable.
Figs.~\ref{fig:erasure_semeval_from} give more detailed information. 
Although changes compared to Table~\ref{tab:text_results_concepts} are marginal on average, they range from $-8.5\%$ to $+13\%$ across domains and models. Grammatical properties, (\emph{Tense}, \emph{SubjNum}, \emph{ObjNum}) have no significant impact, while erasing global syntax information (\emph{TopConst}, \emph{TreeDepth}) improves cross-domain transfer up to $+13\%$, especially from \textit{wikipedia} and \textit{arxiv}. This means that LLMs in general are not good in mimicking complicated syntactic structures, but have no problem with local grammatical categories. Erasing \emph{WC} erasure leads to the largest cross-domain improvement, which means that word semantics produce domain-specific spurious features that harm generalization. There is one outlier: \textit{wikihow}$\rightarrow$\textit{arxiv}; 
we hypothesize that these domains have common word-level features due to many bullet points, numbered lists, and sequential structures in both. 
For cross-model transfer, erasing all three tasks related to error detection in sentence structure (\emph{BShift}, \emph{CoordInv}, \emph{SOMO}) are harmful for ATD performance and robustness; erasing global syntax (\emph{TopConst}, \emph{TreeDepth}) improves performance, while word content (\emph{WC}) leads to contradictory results. 

We conclude that the ability to detect grammatically correct sentences is crucial for robust AI-generated text detection; the difference in global syntax between natural and generated texts is significant, but varies between models and domains, so erasing this information helps generalization, and individual word semantics is a source of spurious features. On the other hand, world-level grammatical categories are captured well by all generators and do not influence ATD performance.

\textbf{Selecting embedding components and heads}. 
To evaluate component removal, 
we use \emph{Reddit} and \emph{Wikipedia} domains from GPT-3D as $\dsearch$, as they have the lowest cross-domain ATD accuracy.
For head selection, we used a lay-off evaluation set with samples of all generators and domains from GPT-3D. We evaluate on GPT-3D and \emph{SemEval}, using the same set of removed heads or components. 
Fig.~\ref{fig:roberta_gpt4} and Table~\ref{tab:main_table} show the results;
transfer to and from \emph{Wikipedia} and \emph{Reddit} subsets has improved. Head selection greatly improves performance on validation domains, achieving the best scores among all the methods.
In cross-task transfer (from GPT-3D to SemEval, Appendix~\ref{sec:appendix_cross_dataset}), component and head removal works better if components are chosen on the same data distribution where the classifier is trained; still, cross-dataset transfer here is generally on par with the baseline.

\begin{figure}[!t]\centering
    \setlength{\tabcolsep}{0pt}
    \def\myhei{0.56\linewidth}

    \begin{tabular}{cc}\centering
        \includegraphics[height=\myhei]{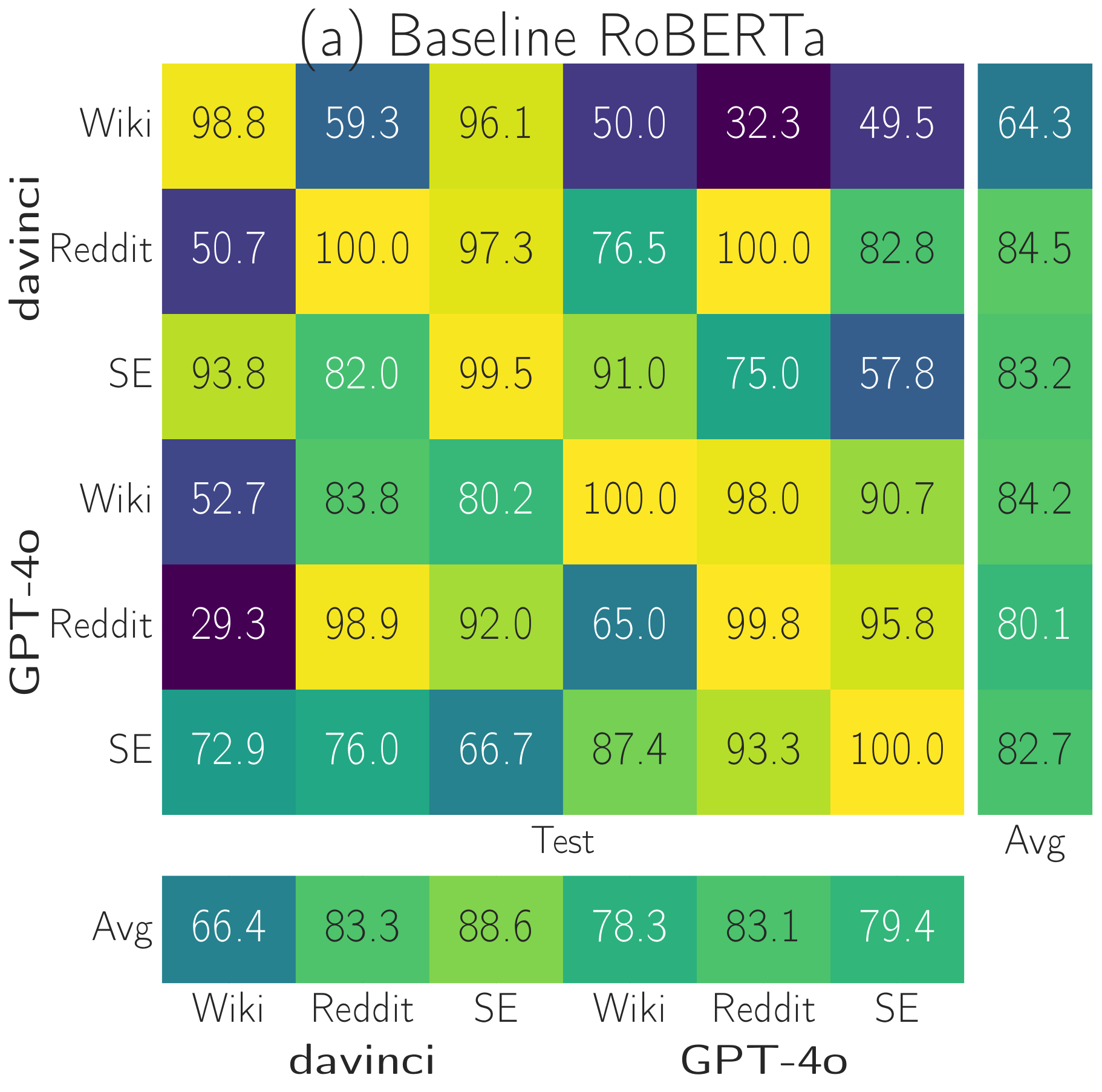} &
        \includegraphics[height=\myhei]{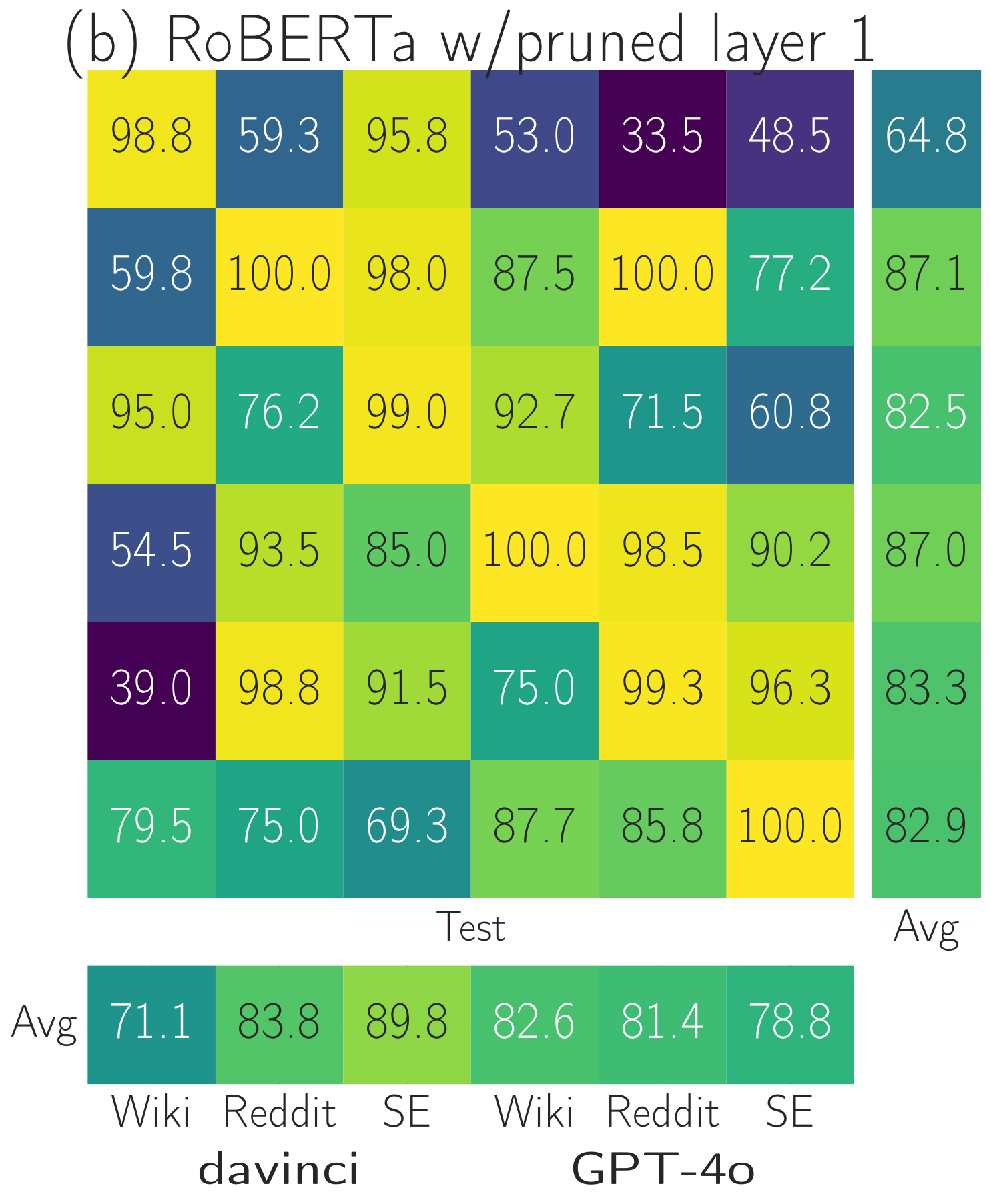} \\
        \includegraphics[height=\myhei]{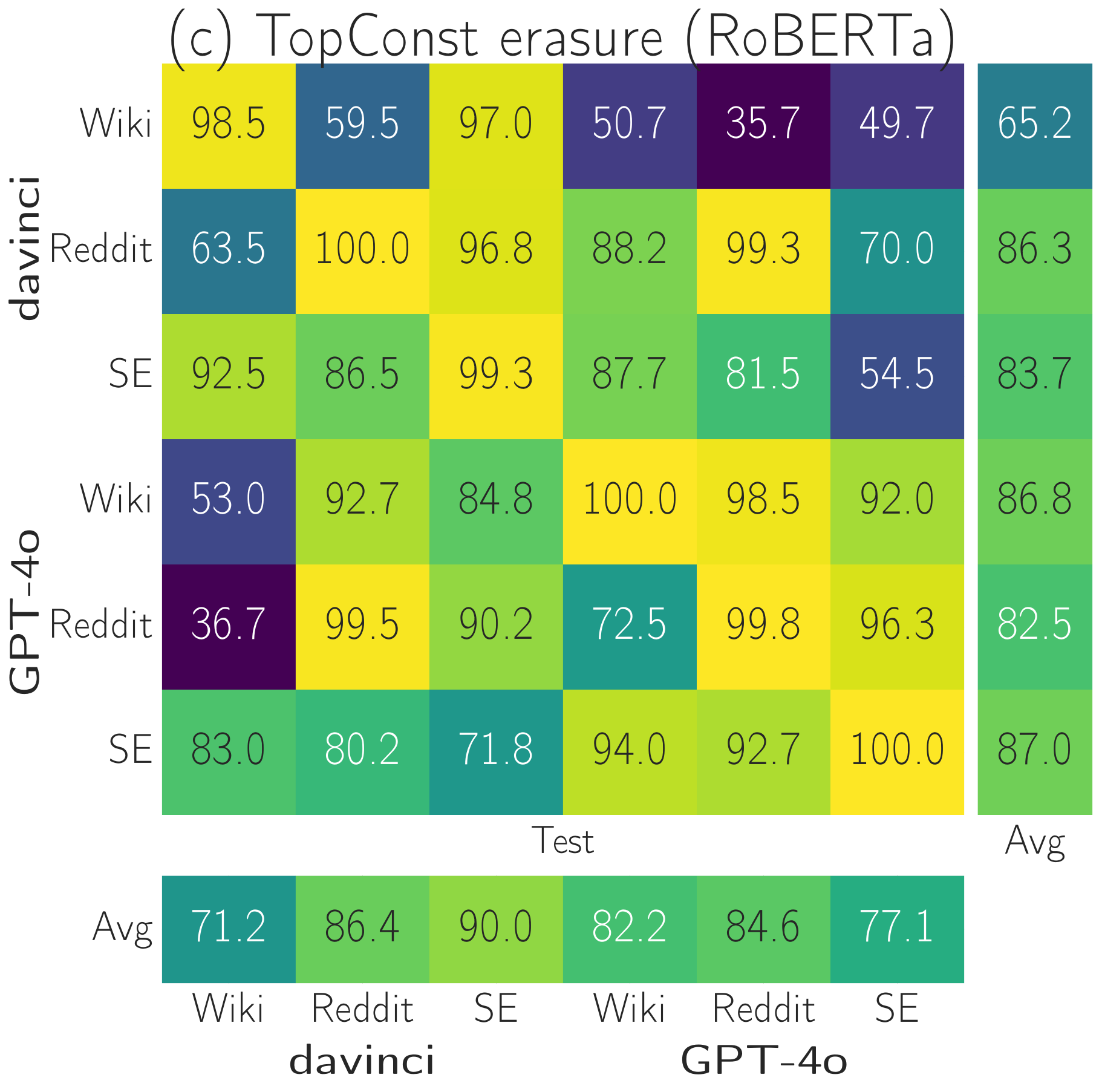}
        &
        \includegraphics[height=\myhei]{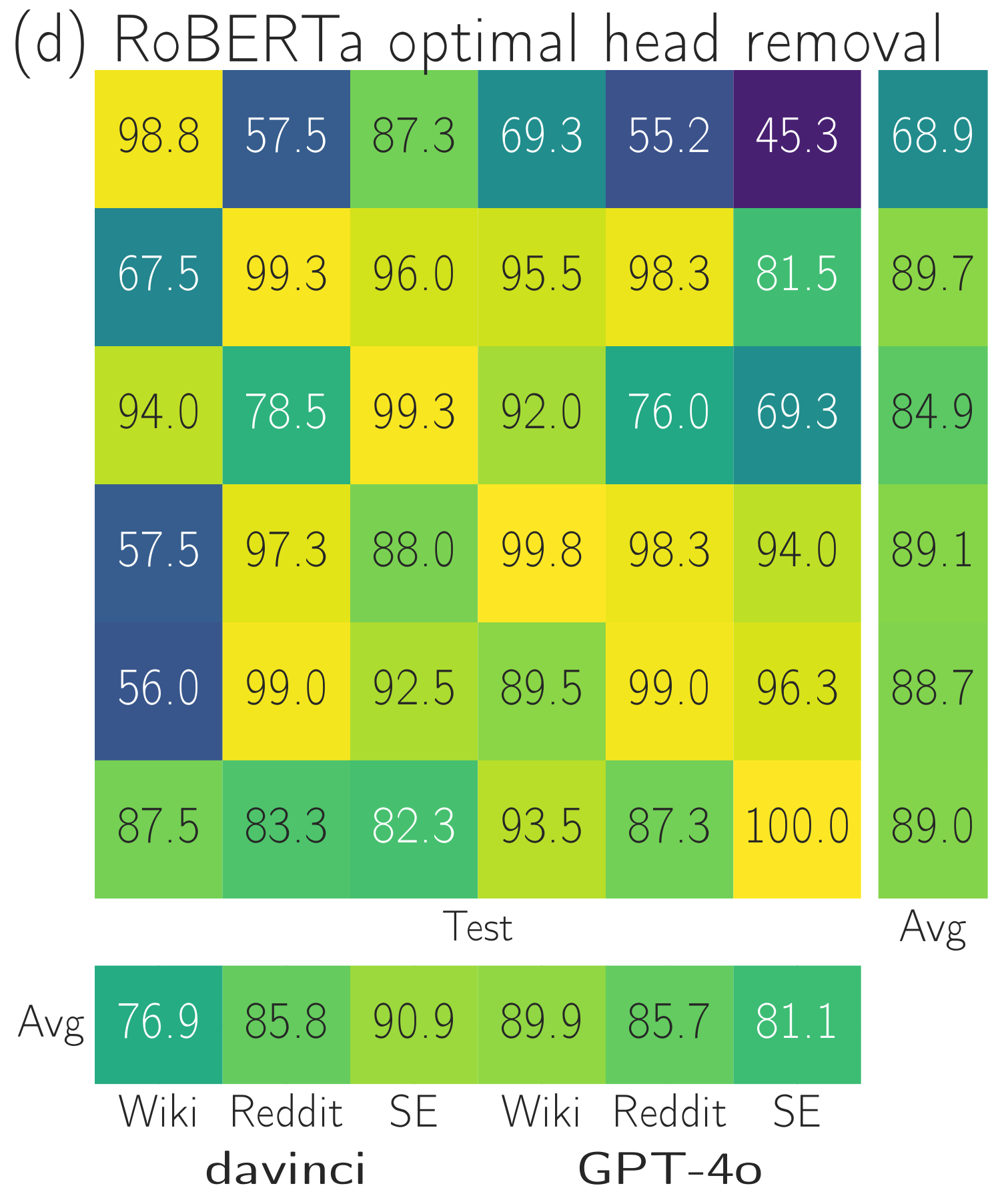} \\
        \includegraphics[height=\myhei]{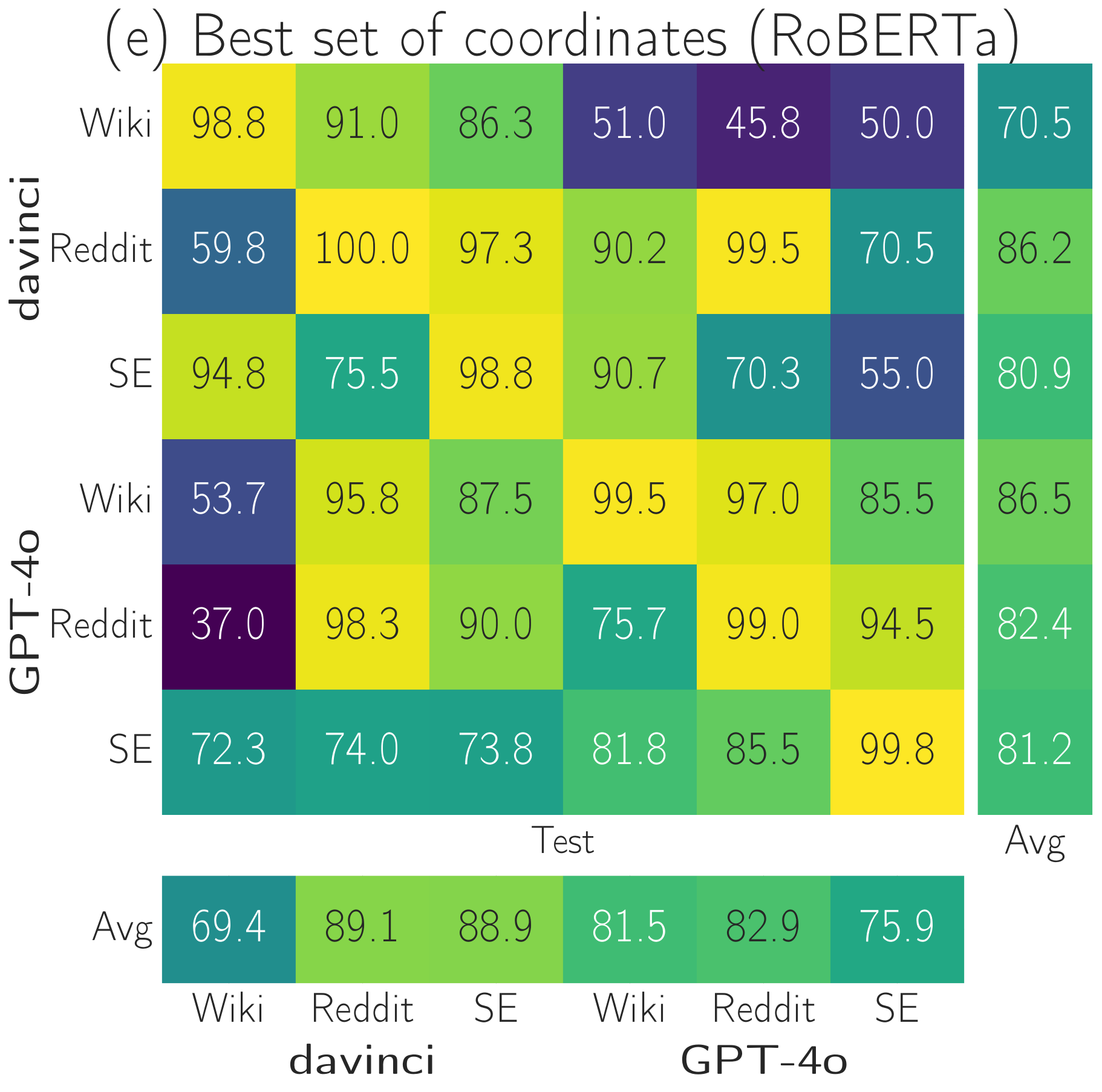} &
        \includegraphics[height=\myhei]{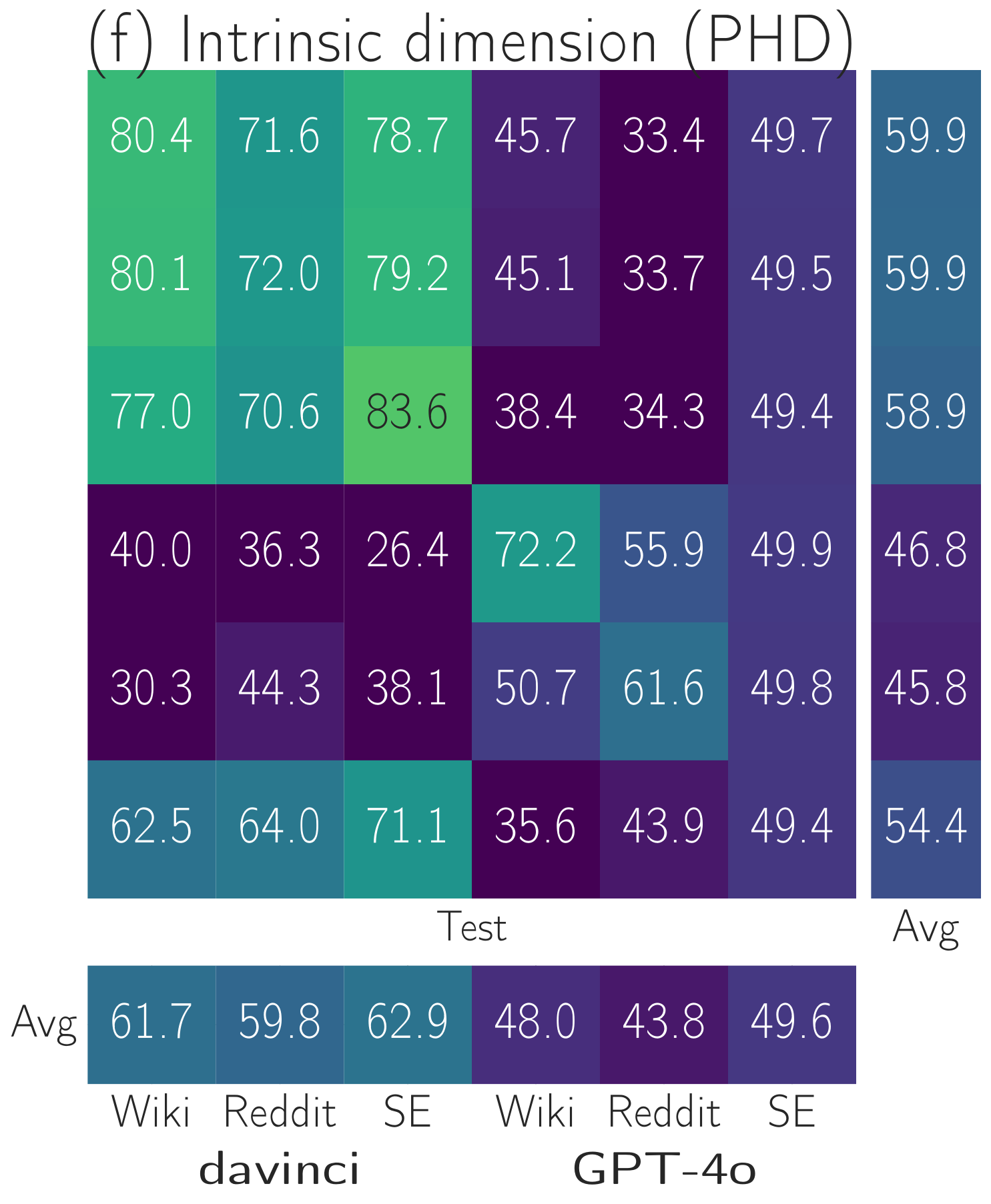}
    \end{tabular}
 
   \caption{Mean accuracy in cross-domain/cross-model ATD on GPT-3D by: (a)~RoBERTa-base, (b)~RoBERTa-base with all attention heads pruned from layer~1, (c) RoBERTa with TopConst concept erasure, (d)~optimal head removal, (e)~best set of coordinates, (f)~classifier based on PHD intrinsic dimensions.}
   \label{fig:roberta_gpt4}
\end{figure}

\begin{figure}[!t]
   \centering
   \includegraphics[width=\linewidth]{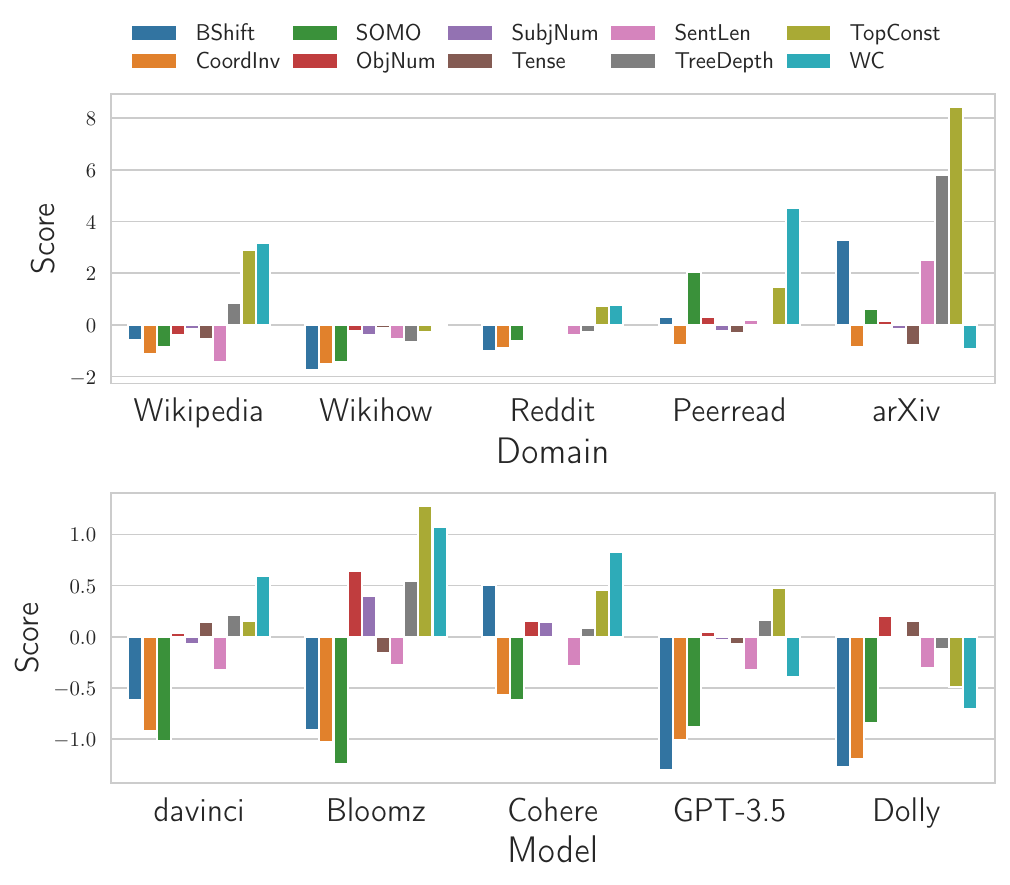}
   \caption{Score change after concept erasure in cross-domain and cross-model settings on \emph{SemEval}.}
   \label{fig:erasure_semeval_from}
\end{figure}

\def\mypwid{.11\linewidth}

\begin{table}[!t]\small\centering
\setlength{\tabcolsep}{1.7pt}
\begin{tabular}{L{.35\linewidth}C{\mypwid}C{\mypwid}C{\mypwid}C{\mypwid}C{\mypwid}}
\toprule
                      & \multicolumn{2}{c}{\textbf{SemEval}} & \multicolumn{3}{c}{\textbf{GPT-3D}} \\
RoBERTa & CD & CM  & CD & CM & CA \\\midrule
Baseline      &   73.0     &     82.8   &   84.1      &   71.0  &  70.1  \\
Layer 1 & \textbf{76.0} &     82.6 &   84.8       &     72.7      & 72.9 \\
Layer 4 &  74.6  & \textbf{83.7} &   84.9      &   72.3   &  72.0 \\
TopConst erased     &   75.1       & 83.1  & \textbf{86.7} & 71.4 & {73.1} \\
TreeDepth erased   &   73.9       &     83.0     &       85.3      &  73.3   &  72.0  \\
\midrule
Selected heads & 74.3 & 80.0  &   86.6  &  \textbf{79.3} & \textbf{79.2}  \\
Selected coordinates &   74.5 & 82.6    & 85.4 & 71.9 &  72.8 \\\bottomrule
\end{tabular}
\caption{Balanced accuracy for OOD classification: cross-domain (CD), cross-model (CM), cross-all (CA). For confidence intervals on SemEval, see Appendix~\ref{sec:confidence_intervals}.}
\label{tab:main_table}
\end{table}

\begin{table*}[!t]\small\centering
\setlength{\tabcolsep}{10pt}
\begin{tabular}{llllllllll}
\toprule
                     & \multicolumn{3}{c}{\textbf{BERT, GPT-3D}} & \multicolumn{3}{c}{\textbf{Phi2, GPT-3D}} & \multicolumn{3}{c}{\textbf{MiniCPM, GPT-3D}} \\
                    & \textbf{CD} & \textbf{CM} & \textbf{CA}
                    & \textbf{CD} & \textbf{CM} & \textbf{CA} & \textbf{CD} & \textbf{CM} & \textbf{CA} \\\midrule
Baseline & 82.4 & 81.9 & 71.1 & 92.2 & \textbf{92.3} & 86.7 & 92.8 & 88.5 & 80.5 \\
Layer 1 & 83.2 & 77.8 & 69.3 & 85.5 & 89.5 & 78.0 & 77.3 & 65.8 & 56.3 \\
Layer 4 & 82.2 & 78.9 & 69.6 & 92.6 & \textbf{92.3} & \textbf{87.2}& 92.0 & 87.0 & 78.0 \\
Selected heads & 85.4 & 81.0 & 73.1 & --- & --- & --- & --- & --- & --- \\
TopConst erased & 83.1 & 81.4 & 70.9 &  91.8 & 91.5 & 86.1 & 92.8 & 87.2 & 80.2 \\
TreeDepth erased  &  84.0 & 83.2 & 71.8 & \textbf{93.3} &  91.8 & 87.0 & \textbf{93.4} & \textbf{88.6} & \textbf{80.6}\\
Selected coordinates & \textbf{92.1} & \textbf{88.0} & \textbf{85.2} & 93.1 & 89.9 & 86.7 & --- & --- & ---\\
\bottomrule
\end{tabular}
\caption{Aggregated OOD scores for BERT, Phi2, and MiniCPM embeddings: cross-domain (CD), cross-model (CM), cross-all (CA). Best results are given in bold.}
\label{tab:main_table_other_models}
\end{table*}

\textbf{Influence of the embedding model}.
RoBERTa is commonly used as the encoder for ATD \cite{krishna2023paraphrasing,solaiman2019release,tulchinskii2023intrinsic}, but we have tested other models as well.
Table~\ref{tab:main_table_other_models} and Figure~\ref{fig:pca_acc_bars} in the Appendix~\ref{sec:appendix_pca} show the results; in all cases, we trained LR on mean-pooled embeddings of the last layer. There is an interesting difference between encoder and decoder-based models: although the quality is very different and correlates with model size, all tested encoders are well suited for our context removal methods (their performance increases, often significantly), while the decoder's behaviour is the opposite. Table~\ref{tab:main_table_other_models} shows the results of subspace removal methods for BERT \cite{DBLP:conf/naacl/DevlinCLT19} and Phi-2 \cite{phi2techreport} embeddings; Phi-2 is larger, so its baseline scores are much higher,
but embedding restriction does not lead to improvements while BERT's quality increases, making the results of these models comparable after component removal despite different model size. 
To test our methods with more resource-efficient smaller LMs, we used the MiniCPM-1B model~\cite{hu2024minicpm}. Table~\ref{tab:main_table_other_models} shows that, as expected, concept erasure yields marginal improvements and other methods do not. In absolute values, MiniCPM is on par with Phi-2 in the cross-domain setup and behind Phi-2 and BERT in cross-model and cross-all settings.

We believe that the different behaviour of our methods reflects the fundamental difference in the embedding space geometry of encoders and decoders caused by limitations of the expressive power of the attention due to the triangular attention mask (e.g., the group of upper triangular matrices does not contain any nontrivial rotations or orthogonal transforms in general). On the other hand, high performance of our methods for relatively small encoder-based models shows that their text representations contain disentangled elementary features learned in pretraining and expressed by separate embedding coordinates, attention heads (i.e., linear terms in input-dependent embedding decompositions), or global directions in the embedding space.

We also report how removing components influences the embedding space geometry. PHD intrinsic dimension has the opposite behaviour in GPT-3 and GPT-4 families: the generalization ability of a PHD-based ATD classifier decreases after removing embedding components (see Appendix~\ref{sec:appendix_phdim}).

\textbf{Probing experiments with restricted embeddings}.
To understand the semantics of the removed components, we performed probing experiments upon restricted embeddings. Namely, we compared the results of a baseline model with those after removing layers or coordinates (a subset selected to optimize ATD robustness) on 10 probing tasks for different linguistic properties (see Section~\ref{sec:data}). 

\begin{table}[!t]\small\centering
\setlength{\tabcolsep}{5pt}
\begin{tabular}{lcccc}\toprule
 & \textbf{BERT} & \multicolumn{3}{c}{\textbf{Removing:}} \\
\textbf{Task} &  & \textbf{Coords} & \textbf{Layer 1} & \textbf{Layer 4} \\\midrule
BShift & 86.9 & 78.3 & 88.6 & 86.5 \\
CoordInv & 64.0 & 56.3 & 62.2 & 62.0 \\
ObjNum & 82.9 & 65.7 & 83.0 & 83.0 \\
SOMO & 64.6 & 60.4 & 64.6 & 65.2 \\
Tense  & 89.2 & 82.7 & 88.7 & 89.1 \\
SentLen & 73.8 & 44.1 & 77.9 & 75.5 \\
SubjNum & 87.3 & 72.7 & 88.3 & 87.5 \\
TopConst & 60.7 & 36.7 & 63.6 & 62.1 \\
TreeDepth &  31.9 & 22.2 & 32.9 & 34.5 \\
WC & 24.5 & 8.4 & 30.2 & 25.4 \\\bottomrule
\end{tabular}

\caption{Probing experiments.}\label{tbl:probing}
\end{table}

Table~\ref{tbl:probing} shows the results. Interestingly, removing the coordinates leads to a dramatic decrease in performance on five tasks, which means that the corresponding properties are almost completely ``erased'' from the embeddings. On the contrary, layer pruning has virtually no influence on any of the tasks, which means that elementary linguistic knowledge is fully kept. It is important to note that the probing tasks in Table~\ref{tbl:probing} are all related to grammar and syntax rather than semantics and style.

\section{Conclusion}\label{sec:concl}

In this work, we aim to improve the robustness of artificial text detectors via linear feature removal from text embeddings. 
We propose three ideas that are extremely easy to implement and achieve stable improvement in robustness averaged across domains and models, up to $14\%$ depending on the text encoder. More importantly, we conclude with the following novel insights from our work.

First, new generation models can completely break detectors; e.g., on the GPT-4 family previous detectors' perform below random, while on the same model classifiers demonstrate very high performance in the cross-domain setting. The reason could be the presence of watermarks in GPT-4 generations; if so, watermarks unknown for ATD developers are dangerous, leading to unpredictable black-box behaviour.

Second, performance with respect to the training subset is often counterintuitive; e.g., a classifier trained on Wikipedia may perform worse than on Reddit, although Wikipedia is considered a cleaner domain, better suited for general-purpose models.

Third, Transformer encoders learn disentangled intrinsic features in coordinates and attention heads, and simple decompositions perform better for ATD than more complex approaches. But this effect is less pronounced for decoder models. We plan to study differences in the geometry of encoder and decoder-based text representations in future work.

Finally, global syntax and sentence complexity is a key point for ATD, but the exact differentiating features are domain- and model-specific, so this information should be ignored. Local grammatical categories do not provide an important signal for ATD. Instead, the classifier should rely on features for detecting various types of inconsistencies.   

\section{Limitations}

In this work we show how state of the art ATD methods may fail, for instance, to transfer to new generative models. Our method increases OOD performance on some generators, but there is no guarantee that this property will be preserved for all future models. Novel pretraining techniques, data collection and processing paradigms, and model architectures can change the picture entirely. Since our method is based on supervised classification, it is not clear which features are actually important for it. It can also lead to unexpected results, especially in the presence of \emph{watermarks}, small changes in data distribution inside each generated sample deliberately injected by generative model developers. We believe that for truly reliable ATD detection, all conclusions should be interpretable, so that a human analyst could inspect the decision. By proposing the concept erasure approach, we have made a step towards interpretable ATD.

We have tested our approaches using relatively small subsets of uni-model or uni-domain data and demonstrated promising quality improvements. Nevertheless, it is still not identical to real-world scenarios, where at least several domains and generators are available in training time, and even more have to be considered during the model's application. One of our objectives in this work has been to propose a novel direction that can significantly improve ATD methods in the future and make them more reliable, but currently it is not yet a fully practical production-ready solution. 

Finally, we do not address the real-word case of post-processed and paraphrased generations, and also texts partially written by humans. For example, if some sentences of this section have been generated by GPT-4o but then partially corrected by the authors, most probably the methods considered in this work would not be able to detect it. We leave this direction for further study.

\section*{Acknowledgements}

The work of Sergey Nikolenko was supported by a grant for research centers in the field of artificial intelligence, provided by the Analytical Center for the Government of the Russian Federation in accordance with the subsidy agreement (agreement identifier 000000D730321P5Q0002) and the agreement with the Ivannikov Institute for System Programming of the Russian Academy of Sciences dated November 2, 2021 No. 70-2021-00142.

\bibliography{main}

\appendix

\section{Residual subspaces for ATD}

\subsection{Formal definitions and theory}\label{sec:appendix-def-and-theory}

In this subsection, we introduce formal definitions and recap some statements from linear algebra that are useful for a better understanding of the geometry and properties of residual subspaces. First, we define the notion of \emph{explained variance} and \emph{relative explained variance} to be able to quantify the properties of residual subspaces.

\def\pr{\mathrm{Pr}}
\def\tr{\mathrm{Tr}}
\def\bmu{\boldsymbol{\mu}}
\def\EV{\mathrm{EV}}
\def\RV{\mathrm{RV}}

\begin{definition}[Subspace explained variance \cite{shen2008sparse, gandelsman2023interpreting}] Let $\mathcal{D}\subset\mathbb{R}^d$, $\mathcal{D}=\{\x_1,\dots,\x_N\}$ be a dataset, and $S\subset\mathbb{R}^d$ is an arbitrary subspace, with $Pr(\x):\mathbb{R}^d\rightarrow S$ being the projection function onto $S$. We call the variance of the projections $Pr(\mathcal{D})$ the \emph{explained variance} of subspace $S$ with respect to $\mathcal{D}$:
\begin{multline*}
    \EV^{\mathcal{D}}(S) = \mathbb{E}_{\mathcal{D}} \| \pr(\x - \mathbb{E}[\x]) \|^2 = \\ 
    =\frac{1}{N}\sum_{\x\in\mathcal{D}} \| \pr(\x) - \pr(\bmu) \|^2, 
\end{multline*}
where $\bmu =\frac{1}{N}\sum_{\x\in\mathcal{D}}\x$.
\end{definition}

If $\bar{X}$ is a matrix of centered data vectors $(\x-\bmu)$ for $\x\in\mathcal{D}$ (row-wise), and $V$ is the $k\times d$ matrix defining an arbitrary basis of the subspace $S$, $S=\langle v_1, \dots, v_k\rangle$, then the explained variance $\mathcal{EV}^{\mathcal{D}}(S)$ can be written in matrix form:
\begin{equation}
\EV^{\mathcal{D}}(S) = \tr(\pr(\bar{X})^T\pr(\bar{X})),
\label{eq:trace}
\end{equation}
where the projection operator $\pr(X)$ can be computed as
\begin{equation}
\pr(\bar{X}) = \bar{X}V^T(VV^T)^{-1}V.
\label{eq:projection}
\end{equation}
In the case of an orthonormal basis, 
$V^TV=\mathbb{I}$, formulas~\eqref{eq:trace} and~\eqref{eq:projection} become a simple decomposition into the sum of component-wise variations:  
\begin{equation}
\EV^{\mathcal{D}}(S) = \sum_{i=1}^{k}\mathcal{V}^\mathcal{D}_i,
\label{eq:orthonormal}
\end{equation}
where $\mathcal{V}^\mathcal{D}_i$ is the variance along the $i$th basis vector.

\textit{Relative explained variance} reflects the relative importance of a subspace by the ratio of the subspace explained and total variance of the data: 
$$
 \RV^\mathcal{D}(S) = \frac{\EV^\mathcal{D}(S)}{\mathrm{Var}(\mathcal{D})}.
$$

For data distributed equally over all directions, it is proportional to the subspace dimension. For example, for $\mathcal{D}\sim\mathcal{N}(\mu,\sigma^2)$ for any subspace $S$
 $$
 \RV^\mathcal{D}(S) = \frac{\dim(S)}{d}.
 $$

\begin{definition} 
A subspace $S$ is called an \emph{$\alpha$-residual subspace} with respect to $\mathcal{D}$ if and only if its relative explained variance is not greater than $\alpha$: 
\begin{equation}
\RV^\mathcal{D}(S)\le\alpha.
\label{eq:alpha-residual}
\end{equation}

\end{definition}

\begin{figure}[!t]
\includegraphics[width=\linewidth]{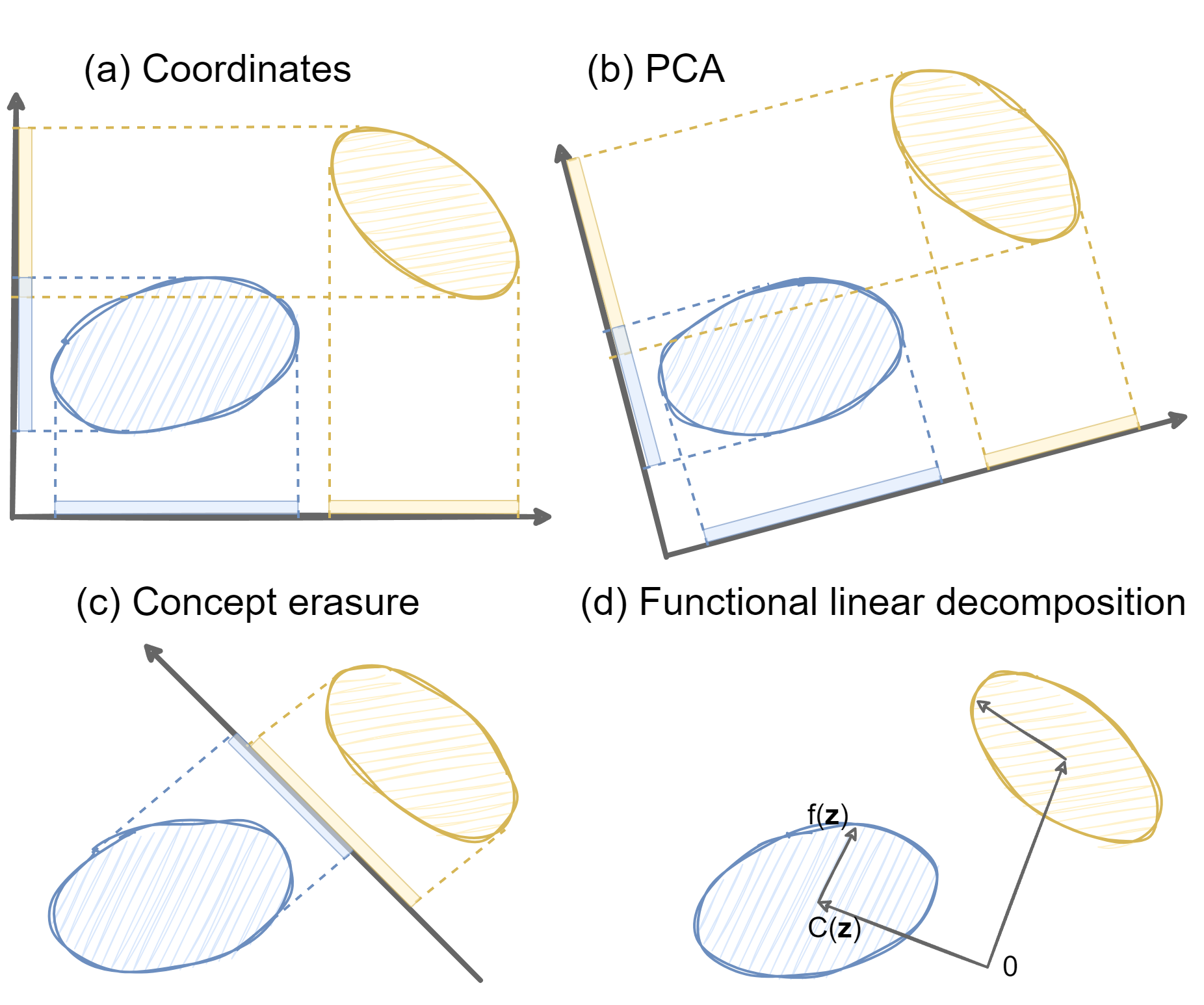}

\caption{Geometric intuition of our approaches.}\label{fig:intuition}
\end{figure}

The simplest way to find residual subspaces for a given $\alpha$ follows from~\eqref{eq:orthonormal}. We can compute the variances $\mathrm{Var}_i$ with respect to each coordinate of the embeddings, and then select the coordinates $\{u_{i_1}, \dots, u_{i_m}\}$ with the smallest variances while their sum does not exceed the desired portion of the total variance. But this method does not guarantee that the required subspace will be found even if it exists for a given dataset. Figure~\ref{fig:intuition} shows the geometric intuition of our approaches; in particular,
the residual subspace, even if it exists, may not be spanned by any subset of the standard basis. The following proposition provides a guaranteed way to find the $\alpha$-residual subspace if it exists.

\def\bu{\mathbf{u}}

\begin{proposition}
Let $\{\bu_1,\dots, \bu_d\}$ be the principal components of a 
dataset $\mathcal{D}$ with corresponding singular values 
$\lambda_1, \dots, \lambda_d$ (in descending order). Then the
explained variance of a subspace spanned by $d-k$ last
principal components $R_k=\langle \bu_{k+1}, \dots, \bu_d\rangle$ is
\begin{equation}
\EV^{\mathcal{D}}(R_k)=\sum_{i=k+1}^{d} \lambda_{i}.
\end{equation}
Moreover, $R_k$ has the minimal explained variance among all $(d-k)$-dimensional subspaces.
\end{proposition}

\begin{proof}
The first statement follows from~\eqref{eq:trace}, taking in account that the trace of a matrix is invariant under the change of the basis. Therefore, we can apply a singular transform to $\bar{X}$ and obtain 
\begin{multline*}
    \tr(\pr_i(\bar{X})^T\pr_i(\bar{X})) = \\ = \tr(\pr_i(\mathrm{diag}(\lambda_1, \dots,\lambda_d))) = \lambda_i.
\end{multline*}

The second statement follows from the Frobenius theorem, which says that for any matrix $\bar{X}$ the projection of its rows to the first $k$ singular components leads to the best rank-$k$ approximation with respect to Frobenius norm:
$$
\langle \bu_1,\dots, \bu_k\rangle = \argmin_{S, \dim S = k} \sum\limits_{\x\in \bar{X}} \| \x - \pr_S(\x) \|^2,
$$
where the sum goes over rows of $\bar{X}$. This can be rewritten in terms of the residual subspace $R=\langle \bu_{k+1},\dots,\bu_d\rangle$, which is unambiguously defined as the orthogonal complement of $S=\langle \bu_1,\dots,\bu_k\rangle$:
\begin{multline*}
\langle \bu_{k+1},\dots,\bu_d\rangle = \argmin_{R, \dim R = d-k} \sum\limits_{\x\in \bar{X}} \| \pr_R(\x) \|^2 \\ =  \argmin_{R, \dim R = d-k} \EV^{\mathcal{D}}(R),
\end{multline*}
which completes the proof.
\end{proof}

As a corollary, PCA allows to find the $\alpha$-residual subspace for a given dataset $\mathcal{D}$, if it exists. Namely, we can select its singular values starting from the least until their relative sum exceeds $\alpha$. Then, the number of components in the sum is equal to the maximal subspace dimension, and the subspace spanned by the corresponding singular vectors provides the necessary subspace.

\subsection{Head-wise decomposition}\label{sec:appendix-headwise-decomposition}

\def\LN{\mathrm{LN}}
\def\MHA{\mathrm{MHA}}
\def\MLP{\mathrm{MLP}}

In our derivation of the form of head-wise flows, we follow the ideas proposed by \citet{gandelsman2023interpreting}. 
In the following, we consider Transformer blocks with post-layer-normalization, such as in BERT and RoBERTa models.
The transformation inside each layer can be written as
\begin{align}
\label{eq:post-ln}
    \hat{\z_l} &= \LN(\z_{l-1}+\MHA(\z_{l-1})), \\
    \z_l &= \LN(\hat{\z_l}+\MLP(\hat{\z_l})),\quad\text{where}
\end{align}
\begin{equation}
\label{eq:LN}
    \LN(\x) = \frac{\x-\bar{\x}}{\|\x-\bar{\x}\|^2},
\end{equation}
and $\bar{\x}=\frac{1}{d}\sum_{i=1}^{d}x_i$ is the mean of the components of a vector $\x$. The numerator of~\eqref{eq:LN} can be rewritten as a linear transform
\begin{equation}
\label{eq:LN_projection}
    \x-\bar{\x} = (\mathbf{I}-\frac{1}{d}\mathbf{1})\x = \Pi \x,
\end{equation}
where $\mathbf{I}$ is the identity matrix, $\mathbf{1}$ is the square matrix consisting of ones, and $d$ is the dimension of $\x$. Note that this transform is in fact an orthogonal projection to the hyperplane defined by the equation $x_1+\dots+x_d=0$. As all projections, $\Pi$ is idempotent:
\begin{equation}
\label{eq:Pi}
\Pi^2=\Pi.
\end{equation}
Applying~\eqref{eq:LN_projection} and~\eqref{eq:Pi} to~\eqref{eq:post-ln}, we can write a layer-wise linear decomposition for post-layer-norm Transformers:
\begin{multline}
    \label{eq:layer-post-ln}
    M(\z)=\alpha(\z)\Pi (\z_0) + \sum\limits_l \beta_l(\z) \Pi (\MLP(\hat{\z_l})) + \\ 
    + \sum\limits_l \gamma_l(\z) \Pi (\MHA(\z_{l-1})) = \\
    = \alpha(\z)\Pi (\z_0) + \sum\limits_l \beta_l(\z) \Pi (\MLP(\hat{\z_l})) + \\ 
    + \sum\limits_l \sum\limits_h \gamma_l(\z) \Pi (A^{l,h}(\z_{l-1})),
\end{multline}
where $\alpha$, $\beta$, $\gamma$ are input-dependent scalars, $\Pi$ is the projection transform~\eqref{eq:LN_projection}, and $A^{l,h}$ denotes attention head $h$ on layer $l$.

\section{Technical details of the experiments}\label{app:technical}

\subsection{Preprocessing and models}

For text preprocessing, we only replaced consecutive spaces, trailing spaces, and a newline characters with one space, as was done by \citet{tulchinskii2023intrinsic}.

For embeddings extraction, we used standard pretrained models from the HuggingFace\footnote{\url{https://huggingface.co/}} library: \texttt{roberta-base} (125M parameters), \texttt{microsoft/phi-2} (2.7B parameters), \texttt{bert-base-uncased} (110M parameters). We use each text sample as an input for chosen model and obtain the resulting embedding from the last layer of this model. We take the mean pooling of that embedding to decrease the dimensionality and get a vector of dimension $768$; this is our text feature vector.
    
For all further experiments with embeddings, we use the logistic regression model from the \emph{scikit-learn}\footnote{\url{https://scikit-learn.org/stable/}} package on the training subset with default parameters: \emph{lbfgs} solver, $L_2$ regularization coefficient $C = 1$, and maximum amount of iterations $\mathrm{max\_iter} = 100$. 

\subsection{Computational resources}

For all of our experiments we used two servers with the following computational resources:
\begin{itemize}
    \item 1 V100 16Gb GPU + 32 CPUs (Intel(R) Xeon(R) Gold 6151), 126GB RAM
    \item 2 V100 16GB GPUs + 64 CPUs (Intel(R) Xeon(R) Gold 6151), 252GB RAM
\end{itemize}

\subsection{Detailed experimental setup on GPT-3D}

For experiments on the GPT-3D dataset, we consider texts generated by either the \emph{davinci} or \emph{GPT-4-o} generator on the $i$th topic from the  list and the corresponding human-written texts on the same topic as one dataset, labeling the generated and human-written texts with ``0'' and ``1'' respectively. We use each text sample as an input for the RoBERTa model and take the mean-pooled embeddings to obtain a vector of dimension $d = 768$; this is our text feature vector.

We split the resulting dataset of these feature vectors into training and test subsets. %
We train logistic regression on the training subset and test the resulting classifier on the test subset of every other generator we have. The resulting accuracy values comprise the $i$th row of our resulting diagram. 
We repeat this process for every considered topic.

\subsection{Greedy search for embedding components}\label{sec:appendix_gready}

Recall that for these experiments, we chose two domains $\{D_1, D_2\}$, and train the classifier on subset $D_1$, using corresponding feature vectors of size $d$. To find the ``harmful'' subspace, we start to remove the components of these feature vectors one-by-one. First, we train the classifier with $0$-th component of the feature vector removed, than with $1$-st component removed and so on, up to the last $d$-th component, remembering, which component removal increases out-of-domain accuracy on $D_2$ the most (or decreases it the least). After finding that most ``harmful'' component, we remove it for good and repeat the process again for the vector of size $d - 1$ to see, which one from the remaining components is the most harmful (or least useful) now. We repeat this process until only one component remains in the feature vector. 

After this, we get a list of the removed components and correseponding accuracy scores. We remember a list of the components, removal of which gives the best OOD accuracy $D_1 \to D_2$. Then we repeat all the same, training classifier on $D_2$ and checking it's performance on $D_1$ to get the list of the components, removal of which gives the best OOD accuracy in the opposite direction, i.e. $D_2 \to D_1$.

Intersection of these lists is the final list of the components that we remove in this method. After removing it, we remain with a union of the best components that need to remain to get the best $D_1 \to D_2$ and $D_2 \to D_1$ scores, as described in Section~\ref{sec:removing}.

\begin{figure}[!t]\centering
\includegraphics[width=\linewidth]{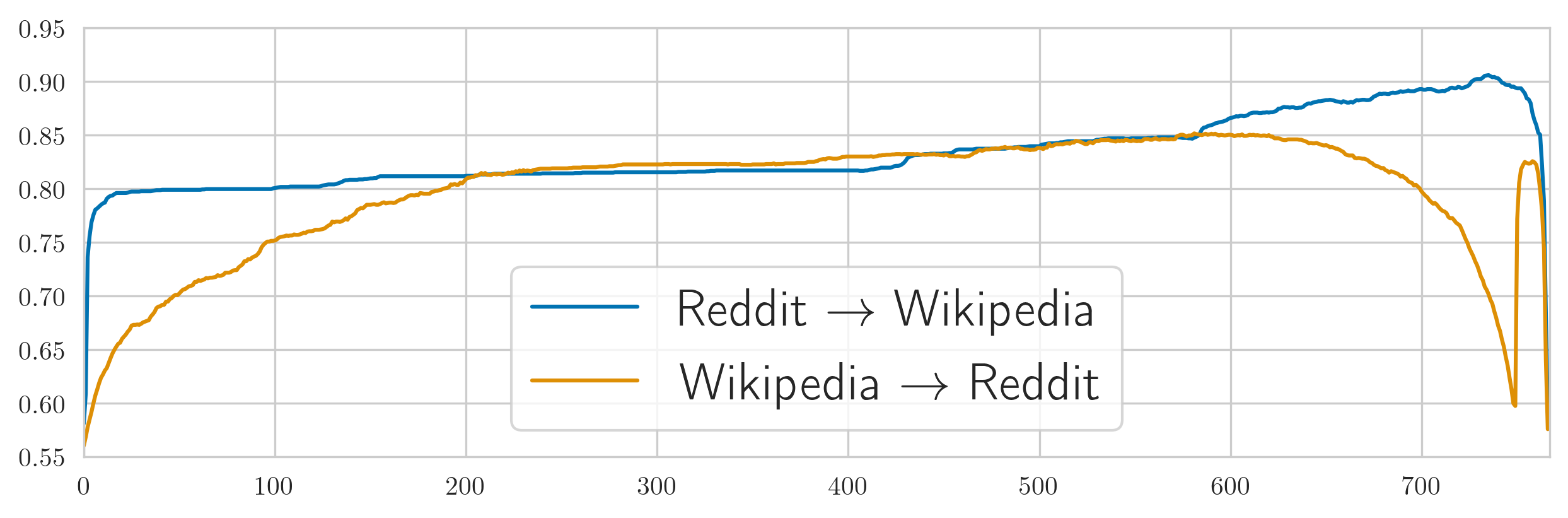}

\caption{Accuracy (vertical axis) as a function of the number of components removed from the RoBERTa embedding (horizontal axis). 
}\label{fig:reddit_wiki}
\end{figure}

\begin{figure}[!t]\centering
\includegraphics[width=\linewidth]{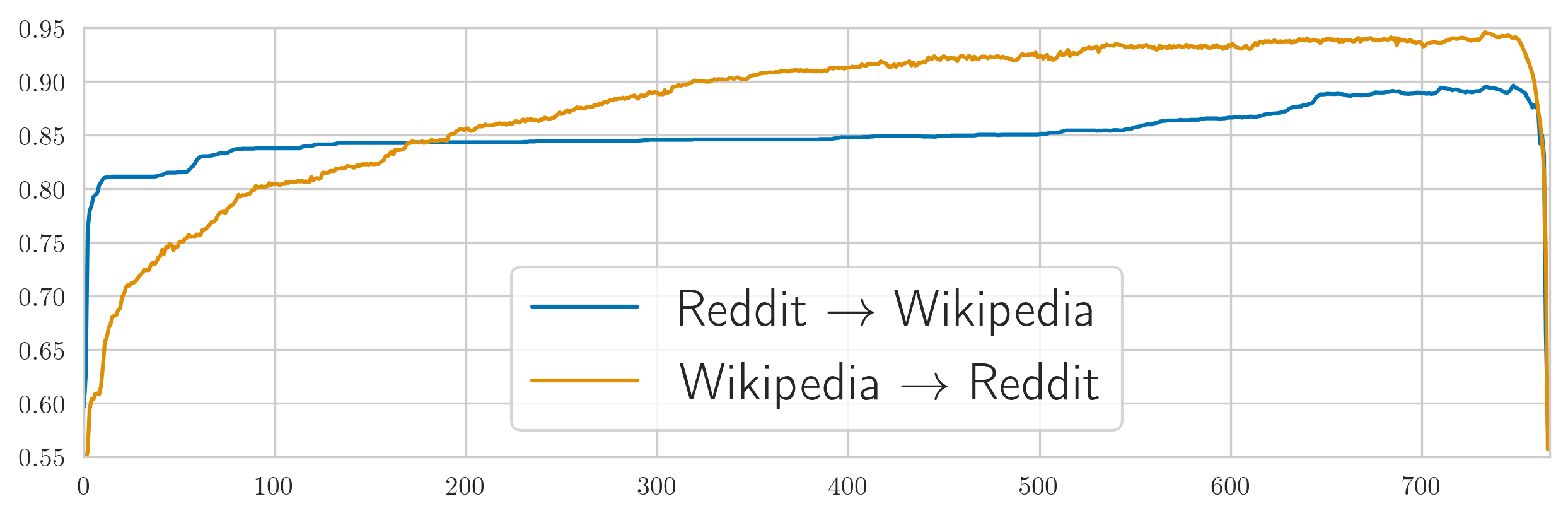}

\caption{Accuracy (vertical axis) as a function of the number of removed components (similar to Fig.~\ref{fig:reddit_wiki}) for data with all symbols except English letters, numbers, and ``!'', ``?'', ``,'', and ``.'' symbols filtered out.}
\label{fig:reddit_wiki_old}
\end{figure}

The resulting scores for greedy search of the embeddings components to remove, in both directions, $S_{\mathrm{Reddit} \to \mathrm{Wikipedia}}$ and $S_{\mathrm{Wikipedia} \to \mathrm{Reddit}}$, are shown in Figure~\ref{fig:reddit_wiki}. We also provide another similar plot in Figure~\ref{fig:reddit_wiki_old} in the setting where all symbols except English letters, numbers, and ``!'', ``?'', ``,'', and ``.'' symbols have been filtered out. This experiment shows that the text preprocessing method can significantly influence the process of choosing the best components.

\subsection{Layer-wise head pruning on GPT-3D dataset, exrtended with GPT-4 generations}

The GPT-3D dataset contains natural and artificially generated texts (by two models: GPT-3.5-davinci-003 and GPT-4-o) in three different domains: \emph{Wikipedia} articles, long-form question answering from \emph{Reddit} (general topics), and \emph{StackExchange} (more technical texts). For each (domain, generating model) pair, the dataset contains an equal number of generated and natural texts from that domain; therefore, classes are balanced in all settings. For each (domain, generating model) pair, we split the data into training and evaluation subsets in the 13:2 ratio. %
None of the evaluation subsets intersect with any of the training subsets.

Although our main track of research on our GPT-3D dataset was conducted using GPT-4-o data, we also generated a small sample of data by the earlier GPT-4 generator. This model is more expensive so the amount of data fit to our budget was not sufficient for a stable evaluation of all the proposed methods; but below we report interesting findings obtained by layer-wise head pruning. Table~\ref{tab:text_results_pruning_GPT4} demonstrates, that in this data-sparse regime the performance of OOD transfer of GPT-4 generations is low, but 1st layer pruning corrects it by as much as $16\%$. This observation does not correspond to the results obtained by GPT-4-o generations. Besides, the quality of cross-model transfer significantly improved. We believe that this observation requires an additional study with a larger GPT-4 dataset.

Below we described the detailed experimental setup for this study.

The experiment was conducted as follows: first, a classifier was trained on data for one (domain, generating model) pair and then evaluated on two other domains with the same generating model; we call this the OOD (out-of-distribution) setting. Then, the classifier is evaluated on all three domains but with a different generating model (Transfer). The results are presented in Table~\ref{tab:text_results_pruning_GPT4}, which reports average accuracy across all domains.

The first row of the table (Full) contains results obtained using the unaltered RoBERTa-base model.
Then we separately prune each layer of attention heads (``turn off'' all 12 attention heads of each layer by zeroing their output); this can be done, e.g., with the \texttt{prune\_heads} method of the RoBERTaModel class from the HuggingFace library. Results for these cases are reported in other rows of Table~\ref{tab:text_results_pruning_GPT4}.

Table~\ref{tab:text_results_pruning_GPT4} shows full results across the layers, indicating that pruning the lower layers of the model, especially Layer 0, yields better results.

\begin{table}[!t]
    \centering\setlength{\tabcolsep}{2pt}
    \resizebox{\linewidth}{!}{\setstretch{0.95}
    \begin{tabular}{lcccc}
        \toprule
             &    davinci & GPT-4 & davinci to GPT-4 & GPT-4 to davinci \\
            &    OOD & OOD              & transfer    & transfer \\
        \midrule
        Full model   & 81.3  & 64.3 & 66.4 & 70.4 \\
        \midrule
        \textbf{Pruned layer} & &  &  & \\
        \#0  & {83.2} & \textbf{80.1} & \textbf{80.0} & \textbf{83.2} \\
        \#1  & 83.4 & \underline{78.8} & \underline{74.7} & \underline{79.2} \\
        \#2  & 82.1 & \underline{78.8} & 72.7 & 77.4 \\
        \#3  & 81.8 & 82.0 & 73.6 & 78.1 \\
        \#4  & 83.4 & 79.6 & 71.6 & 76.4 \\
        \#5  & 82.2 & 78.1 & 72.9 & 75.6 \\
        \#6  &  \underline{84.0} & 76.3 & 72.3 & 74.4 \\
        \#7  & 82.8 & 75.4 & 70.1 & 74.6 \\
        \#8  & 82.8 & 72.1 & 68.5 & 73.4 \\
        \#9  & 83.2 & 73.2 & 68.7 & 71.4 \\
        \#10 & 83.1 & 71.0 & 68.1 & 72.8 \\
        \#11 & \textbf{86.6} & 68.2 & 67.3 & 71.7 \\
        \bottomrule
    \end{tabular}
    }
    \caption{Average accuracy of artificial text detection over three domains (Wikipedia, Reddit, StackExchange) and two generating models (GPT3.5-davinci and GPT4). Detector is trained on one domain against one generator and evaluated on other domains (OOD) and on all domains against unseen generating model (transfer). Best results are given in bold, runner-ups are underlind.}
    \label{tab:text_results_pruning_GPT4}
\end{table}

\subsection{Concept erasure on SemEval}

Table~\ref{tab:text_results_concepts} reports detailed results on concept erasure on the SemEval dataset. For concept erasure we use an open-source implementation\footnote{\url{https://github.com/EleutherAI/concept-erasure}}.

\begin{table}[!t]\small\centering
\setlength{\tabcolsep}{3pt}
{
\begin{tabular}{lcccccc}
\toprule
        & \multicolumn{3}{c}{Cross-domain}              & \multicolumn{3}{c}{Cross-model}              \\
        & Avg           & Max $\uparrow$  & Max $\downarrow$  & Avg           & Max $\uparrow$ & Max $\downarrow$  \\
        \midrule
Roberta & 73.0          & -             & -             & 82.8          & -            & -             \\
Bshift    & 73.0          & 6.4           & -6.8          & 82.2          & 1.5          & -2.6          \\
CoordInv  & 72.1          & 1.1           & -3.7          & 82.1          & 0.9          & -3.4          \\
ObjNum    & 72.9          & 0.9           & -1.5          & 83.0          & 0.7          & \textbf{-0.0} \\
SOMO      & 72.9          & 6.8           & -3.8          & 82.1          & 0.6          & -4.1          \\
Tense     & 72.7          & 0.4           & -1.6          & 82.8          & 1.0          & -0.4          \\
SentLen   & 73.0          & 4.1           & -3.0          & 82.6          & 0.2          & -1.2          \\
SubjNum   & 72.8          & 0.4           & -1.6          & 82.9          & 0.5          & -0.4          \\
TopConst  & \textbf{75.1} & \textbf{12.6} & -1.8          & \textbf{83.1} & 2.2          & -0.9          \\
TreeDepth & 73.9          & 12.1          & \textbf{-1.4} & 83.0          & 1.0          & -0.3          \\
WC        & 74.1          & 11.0          & -8.5          & 83.0          & \textbf{2.9} & -2.9     \\
\bottomrule
\end{tabular}
}
\caption{Balanced accuracy results for out-of-domain classification for different erased concepts on SemEval}
\label{tab:text_results_concepts}
\end{table}

\section{Cross-dataset transfer}
\label{sec:appendix_cross_dataset}

\begin{table}[!t]\small\centering
\setlength{\tabcolsep}{1.7pt}
\begin{tabular}{L{.3\linewidth}C{\mypwid}C{\mypwid}C{\mypwid}C{\mypwid}C{\mypwid}}
\toprule
                      & \multicolumn{2}{c}{\textbf{SemEval}} & \multicolumn{3}{c}{\textbf{GPT-3D}} \\
RoBERTa & CD & CM  & CD & CM & CA \\\midrule
Baseline      &   73.0 / 76.4*     &     82.8 / 76.3*    &   84.1      &   71.0  &  70.1  \\\midrule
Selected heads       &   74.3 / 75.6*    &     80.0 / 75.4* &   86.6      &    79.3    &  79.2  \\\midrule
Selected coordinates &   74.5 / 75.4*  &     82.6 / 75.3*    & 85.4 & 71.9 &  72.8 \\\bottomrule
\end{tabular}
\caption{Balanced accuracy for OOD classification: cross-domain (CD), cross-model (CM), cross-all (CA).  Numbers with asterisks correspond to cross-dataset transfer. }
\label{tab:cross_dataset}
\end{table}

Table~\ref{tab:cross_dataset}
compares the classifiers trained on SemEval dataset with the same setup trained on GPT-3D data, but tested on SemEval. Surprisingly, in cross-domain transfer heads and coordinates selection on GPT-3D leads to an improvement of the performance on SemEval. However, the cross-model performance degrades.

\section{Assessing combination of methods}
\label{app:combination}

To investigate the effectiveness of combining methods, we conducted experiments where multiple techniques were applied simultaneously. The results, presented in Table~\ref{tab:combination}, show that the combined approach does not result in any significant improvement. The joint outcomes are either worse or approximately the same as the best individual component. Removing multiple layers simultaneously is particularly detrimental, whereas concept removal can be combined with other methods more effectively.

\begin{table}[!t]\small\centering
\setlength{\tabcolsep}{3pt}
{
\begin{tabular}{lllll}
\toprule
Method    & CM   & Combination        & CM   &  \\
\midrule
Baseline  & 81.9 & L1 + L4            & 73.4 &  \\
Layer 1   & 77.8 & L1 + L4            & 73.4 &  \\
Layer 4   & 78.9 & TreeDepth+TopConst & 83.0 &  \\
TopConst  & 81.4 & L1+TreeDepth       & 78.4 &  \\
TreeDepth & 83.2 & TreeDepth+Coord    & 89.5 &  \\
Coord     & 89.5 & L1+TreeDepth+Coord & 88.2 &  \\
\bottomrule
\end{tabular}
}
\caption{Result for BERT model, GPT-3D dataset, cross-model setup.}
\label{tab:combination}
\end{table}

\section{Confidence intervals for SemEval}
\label{sec:confidence_intervals}

Since the SemEval dataset is more diverse than GPT-3D, we present the results using averaged statistics. For more detailed information, we report confidence intervals for accuracy changes on the SemEval dataset in Table~\ref{tab:confidence}. We observe that these intervals are predominantly positive, showing improvements of up to 6\% in cross-domain setups.

\begin{table}[!t]\small\centering
\setlength{\tabcolsep}{3pt}
{
\begin{tabular}{lll}
\toprule
          & \textbf{Cross-Domain} & \textbf{Cross-Model} \\
\midrule
Layer 1   & (0.77, 6.65)          & (-1.47, 0.68)        \\
Layer 4   & (0.88, 3.18)          & (0.4 , 1.68)         \\
TopConst  & (0.55, 4.10)          & (-0.06, 0.7)         \\
TreeDepth & (-0.24, 2.62)         & (-0.07, 0.31)        \\
Coords    & (-0.80 , 5.76)        & (-0.98, 1.19)        \\
\bottomrule
\end{tabular}
}
\caption{Confidence intervals for accuracy changes in SemEval using RoBERTa model}
\label{tab:confidence}
\end{table}

\section{Removing ``bad'' outliers and how it influences the geometry of embeddings}

\begin{figure}[!t]\centering
\includegraphics[width=\linewidth]{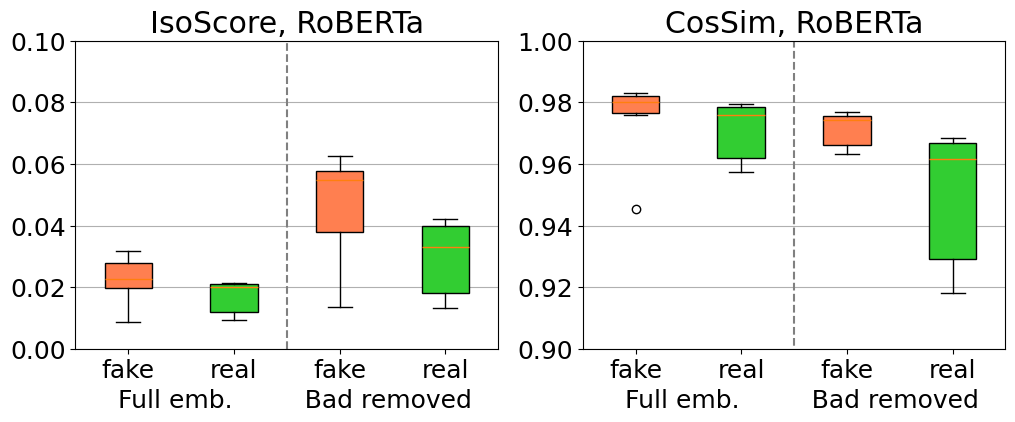}

\caption{IsoScore and cosine similarity of the RoBERTa embeddings before and after removing their ``bad'' components; the embeddings were calculated on GPT-3D dataset.}\label{fig:iso_RoBERTa}
\end{figure}

Previous studies have shown that some dimensions skew the embedding space greatly and have a dramatic influence on its geometry. In particular, \citet{timkey-van-schijndel-2021-bark} have shown that the embeddings of BERT, RoBERTa, and some other Transformer-based models lie in a narrow cone. To show this, they use the mean cosine similarity of the embeddings: if the cosine similarity of all embeddings is high, it means that they are similar to each other along some dimensions; the larger the average cosine similarity, the less isotropic the embedding space is.

\citet{rudman-etal-2022-isoscore} introduced a more complex tool for measuring the anisotropy of the embedding space: IsoScore. The fundamental motivation for IsoScore is that it roughly reflects the fraction of dimensions uniformly utilized by a given point cloud. According to the authors' estimation, less than 20\% of dimensions of the BERT model embedding space are utilized uniformly. Larger IsoScore values correspond to more isotropic embedding spaces.

Figure~\ref{fig:iso_RoBERTa} shows how removing the components that are ``bad'' for cross-domain and cross-model generalization abilities influences the IsoScore and cosine similarity scores for RoBERTa embeddings.

We see that after removing these ``bad'' dimensions, the embeddings of fake and real texts change their isotropy in different rates, but both become more {isotropic} in general. Based on this observation, we hypothesize that the isotropy of the embedding space can be connected to the model's generalization abilities; we leave testing this hypothesis for future research.

\section{Components removal and PHD}
\label{sec:appendix_phdim}

We conducted additional experiments to evaluate the influence of removing embedding components (selected with the greedy search outlined in Section ~\ref{sec:removing} Subspace removing methods) in the RoBERTa and BERT models on the cross-domain and cross-model generalization abilities of the persistent homological fractal intrinsic dimensionality-based method. Figure~\ref{fig:phdim_gpt3d} shows a consistent decrease in accuracy for both cross-model and cross-domain ATD as components are being removed. Such removal typically reduces the intrinsic dimensionality of human-written texts, hence degrading the discriminative power of linear classifiers in ATD. 

An interesting observation is that the PHD of a newer generation LLM (GPT-4o) is higher than that of human-written texts, while the PHD of the older generation (GPT-3.5-davinci) is lower that of human-written texts. This may explain the poor generalization ability between the models on GPT-3.5-davinci and GPT-4o. See Figure~\ref{fig:phdim_statistics} for details.

\begin{figure}[!t]
    \setlength{\tabcolsep}{0pt}
    \def\mywid{0.49\linewidth}
    \begin{tabular}{cc}\centering
        \includegraphics[width=\mywid]{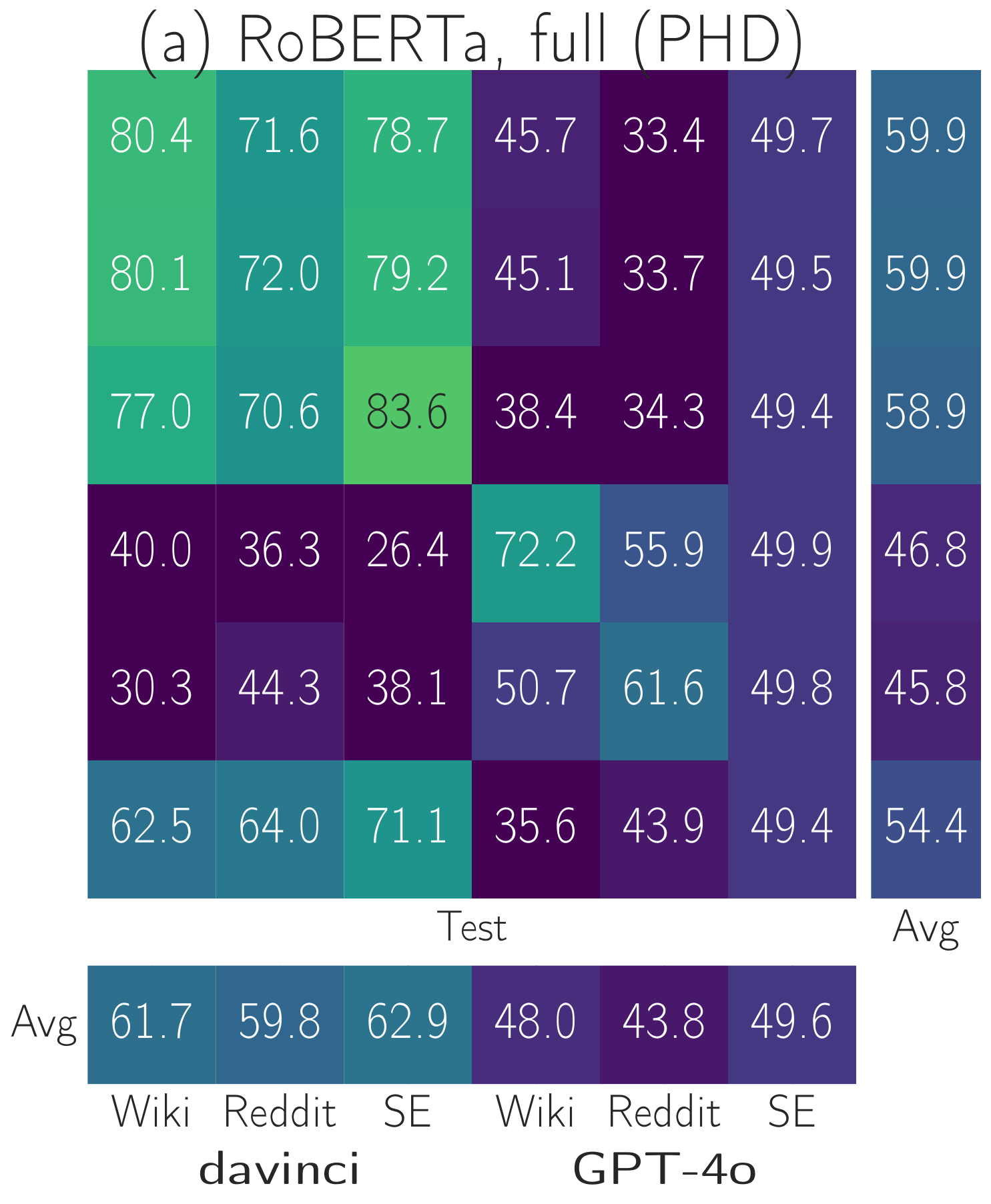} &
        \includegraphics[width=\mywid]{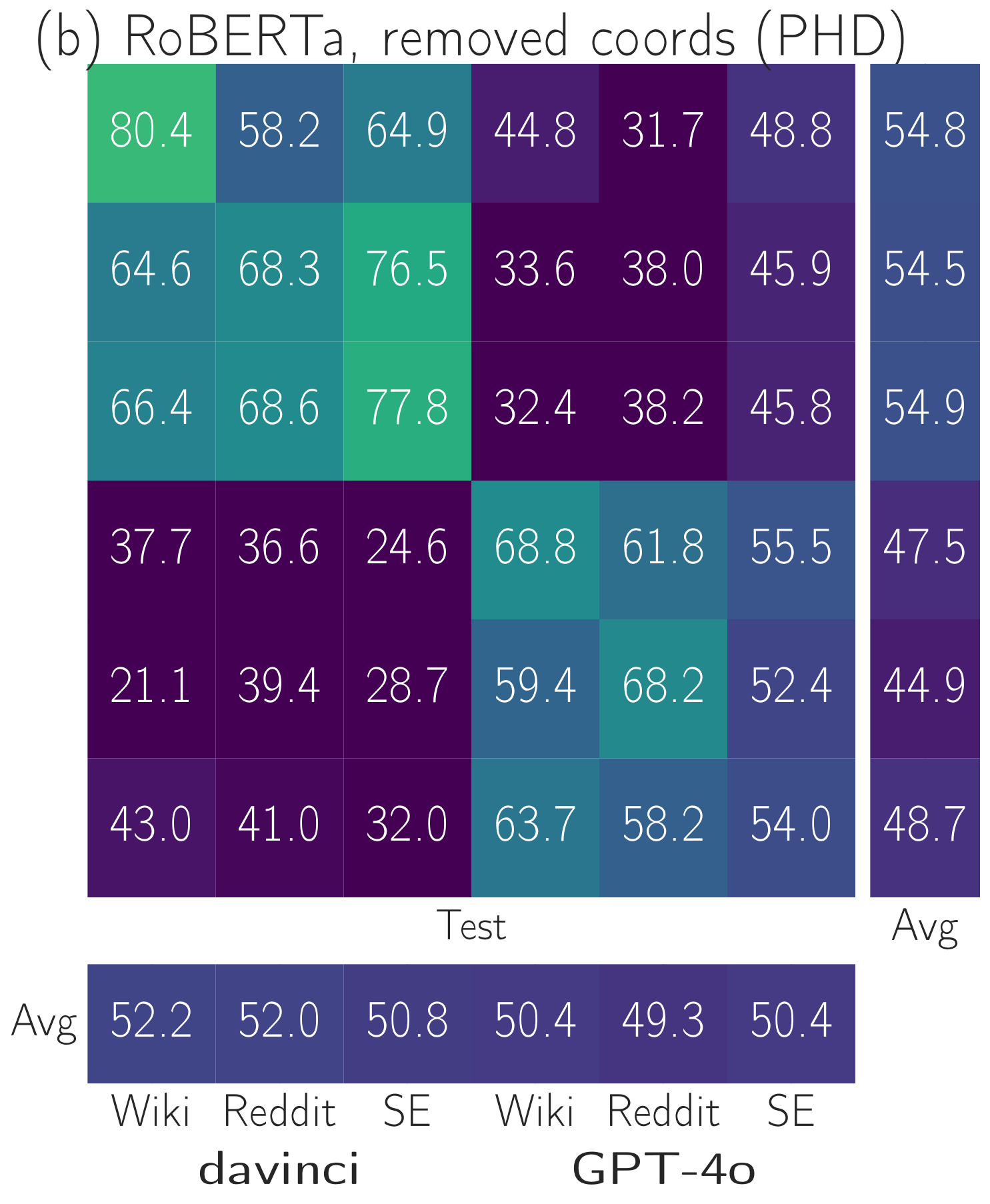} \\
        \includegraphics[width=\mywid]{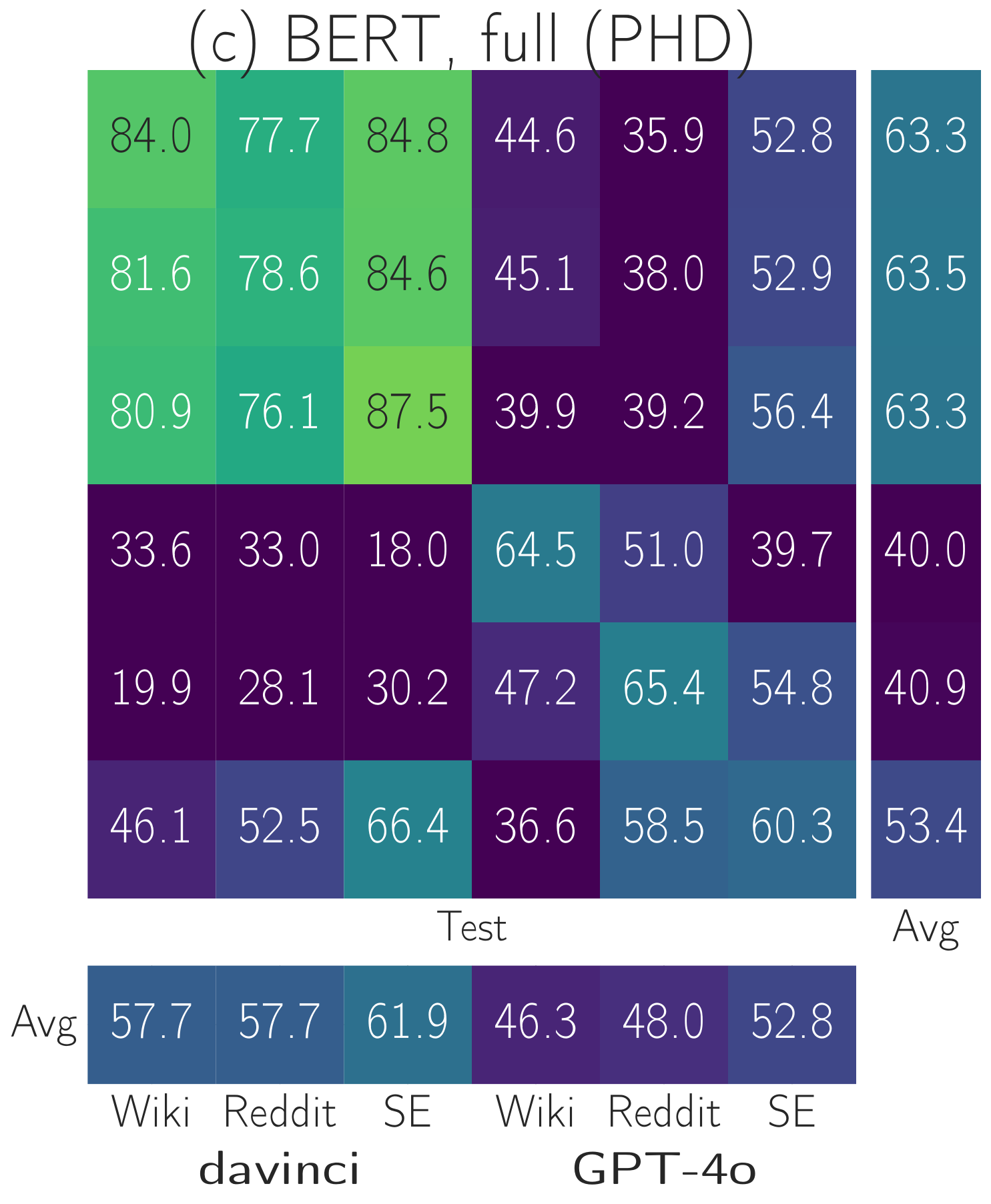} &
        \includegraphics[width=\mywid]{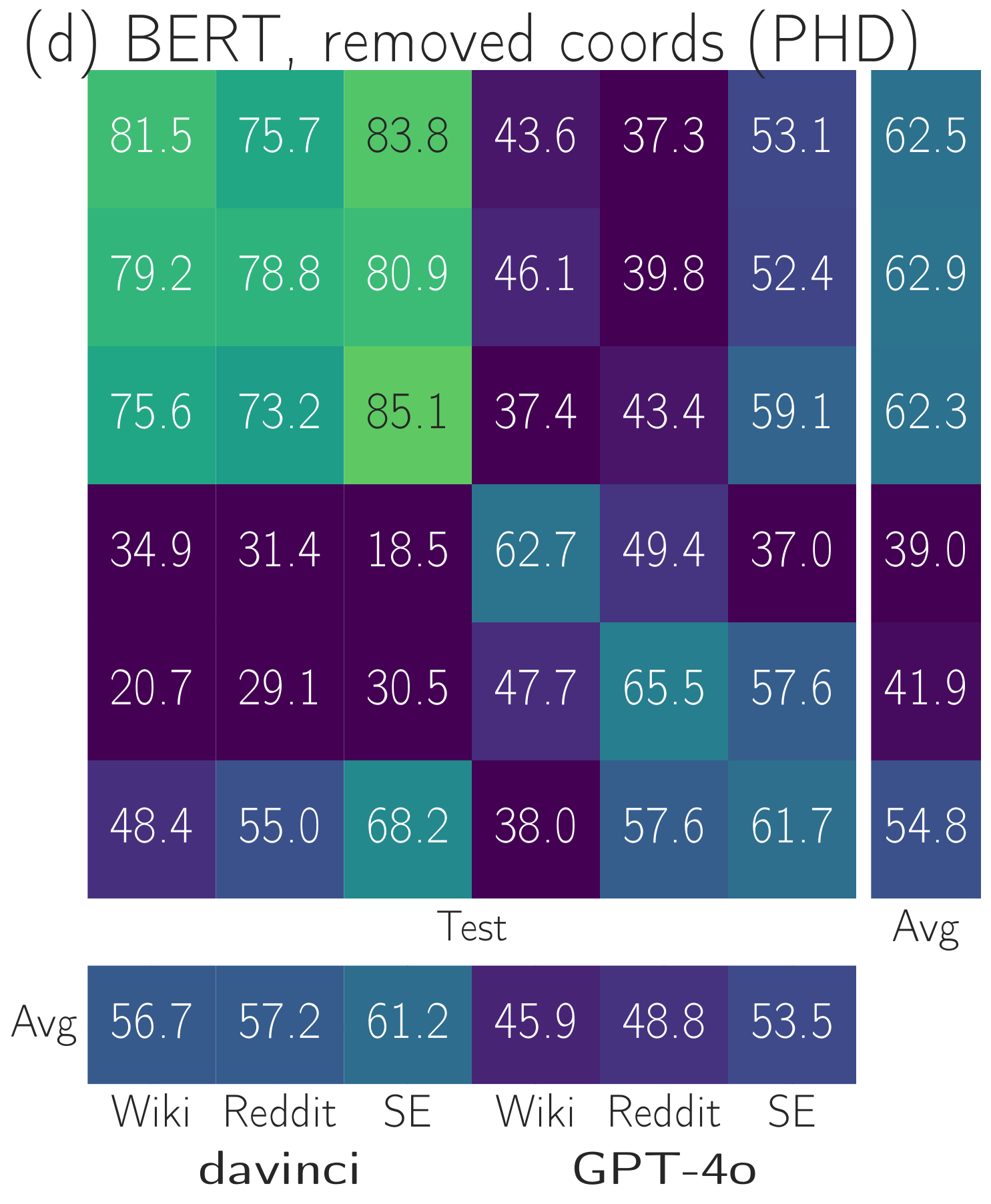}
    \end{tabular}
 
   \caption{PHD-based logistic regression accuracy before and after components removal, mean accuracy in cross-domain/cross-model ATD on GPT-3D: (a)~RoBERTa, full embeddings, (b)~RoBERTa after components removal, (c) BERT, full embeddings, (d) BERT, after components removal.}
   \label{fig:phdim_gpt3d}
\end{figure}

\begin{figure*}[!t]
   \centering
   \includegraphics[width=\textwidth]{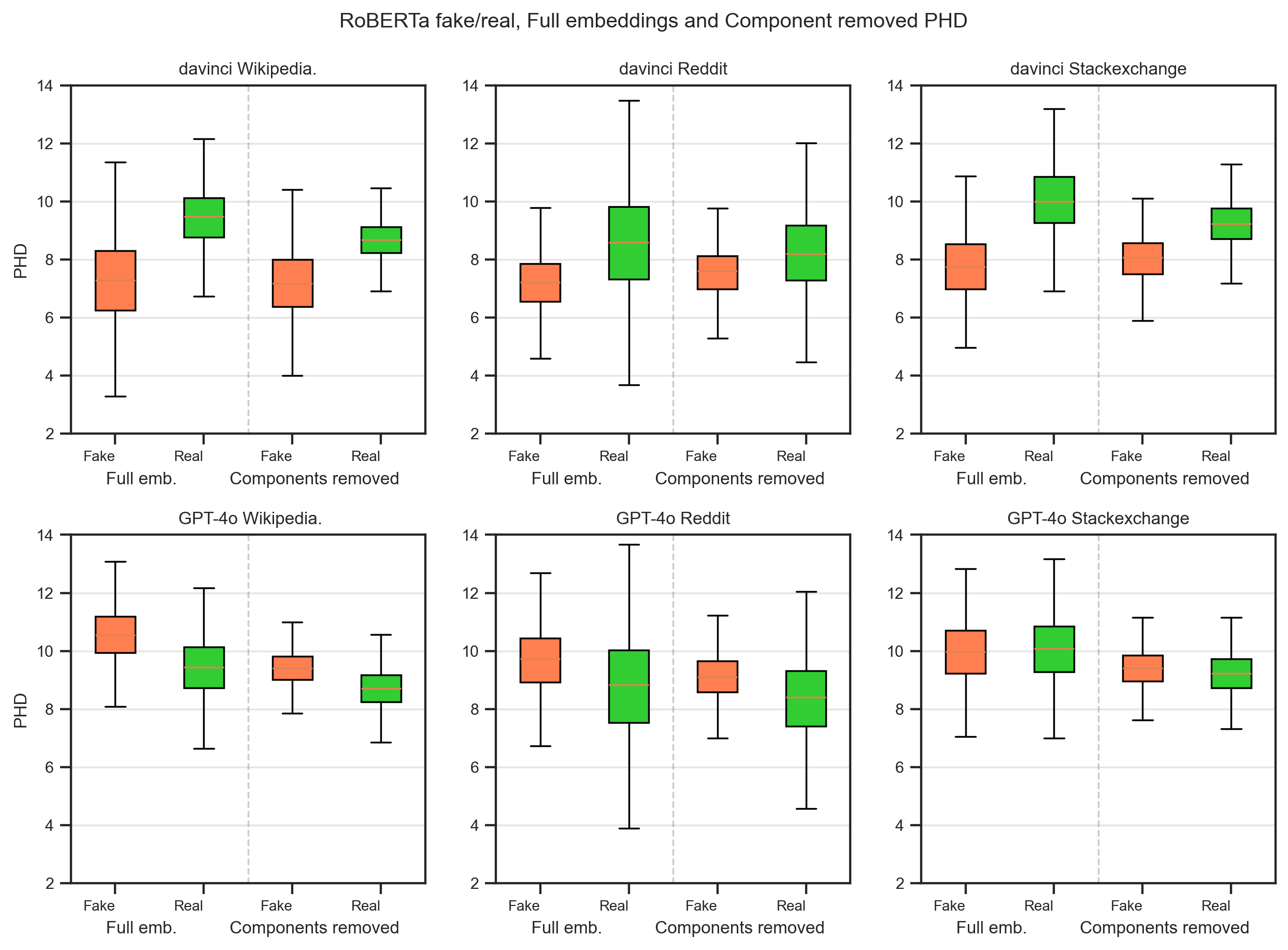}
   \caption{PHD of RoBERTa full embeddings and embeddings after component removal for real/fake texts from the GPT-3D dataset.}
   \label{fig:phdim_statistics}
\end{figure*}

\begin{figure*}[!t]
   \centering
   \includegraphics[width=\textwidth]{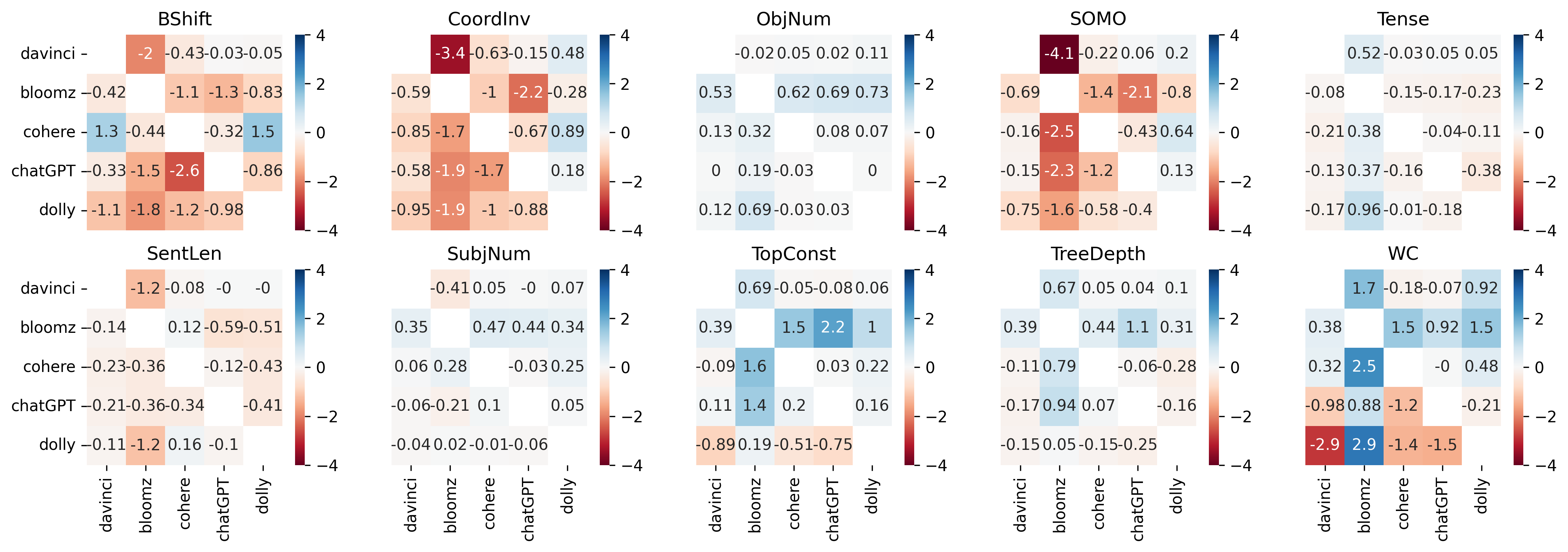}
   \caption{Concept erasure, cross-model setting}
   \label{fig:concept_erasure_cm}
\end{figure*}

\begin{figure*}[!t]
   \centering
   \includegraphics[width=\textwidth]{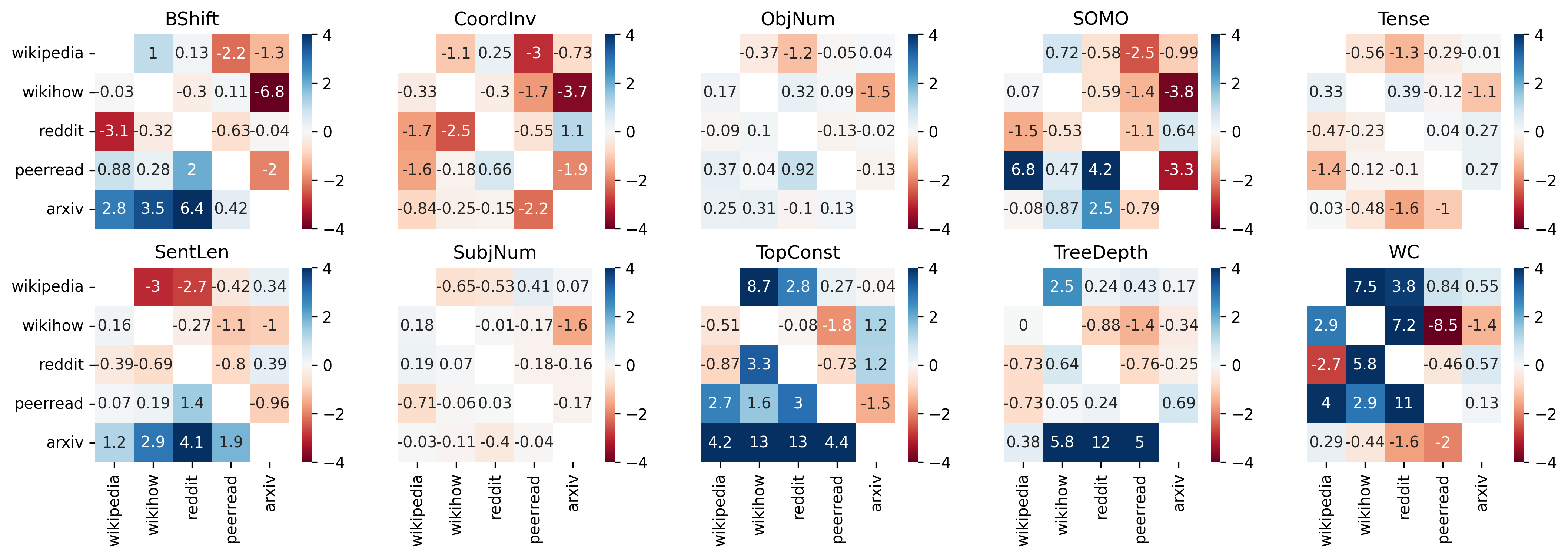}
   \caption{Concept erasure, cross-domain setting}
   \label{fig:concept_erasure_cd}
\end{figure*}

\begin{figure*}[!t]
   \centering
   \includegraphics[width=0.45\textwidth]{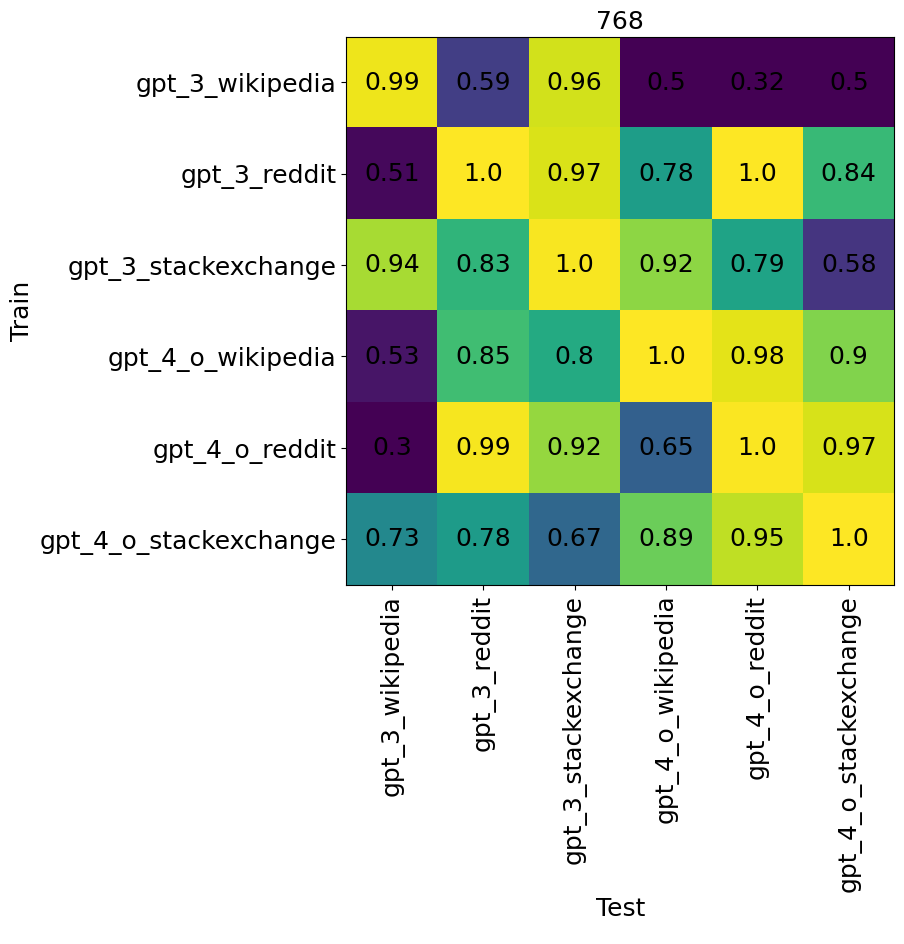}
      \includegraphics[width=0.45\textwidth]{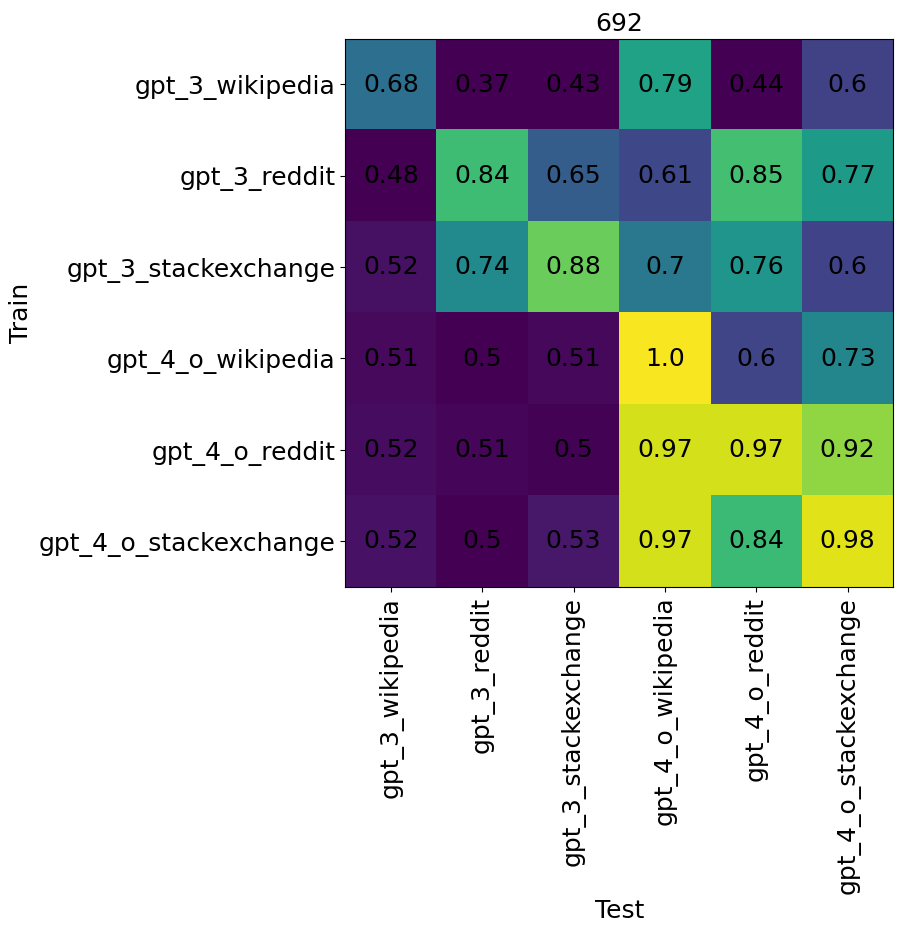}
      \includegraphics[width=0.45\textwidth]{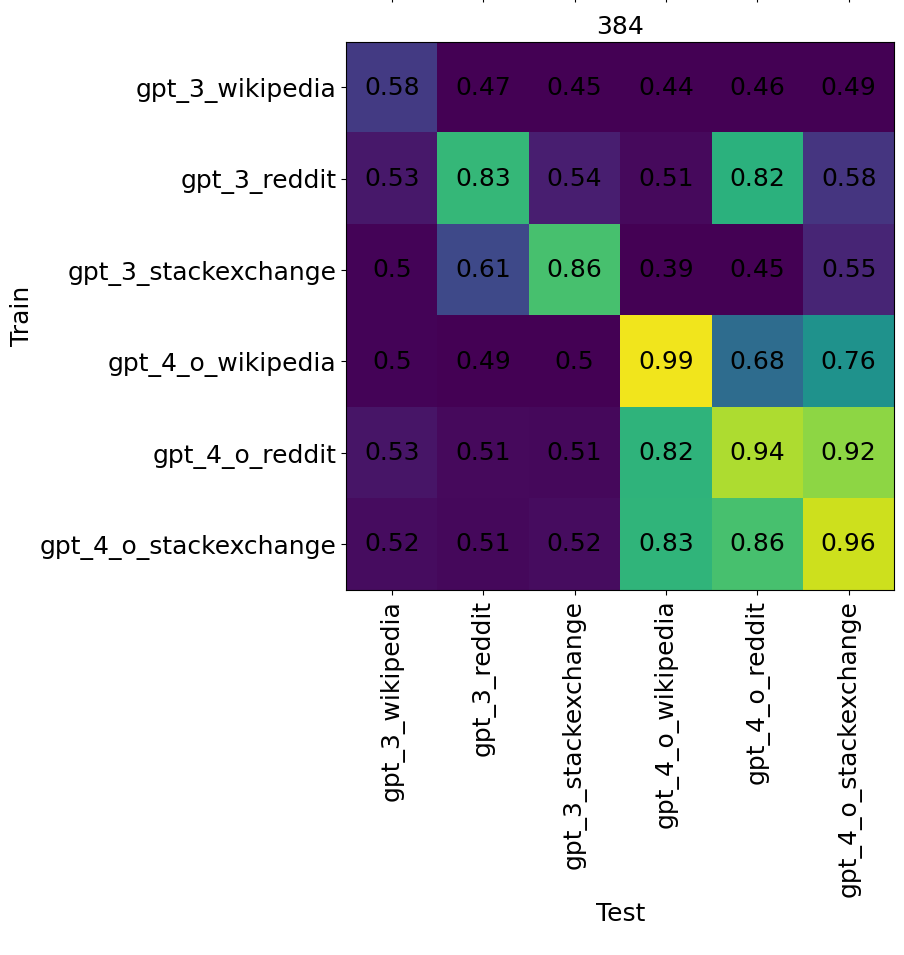}
         \includegraphics[width=0.45\textwidth]{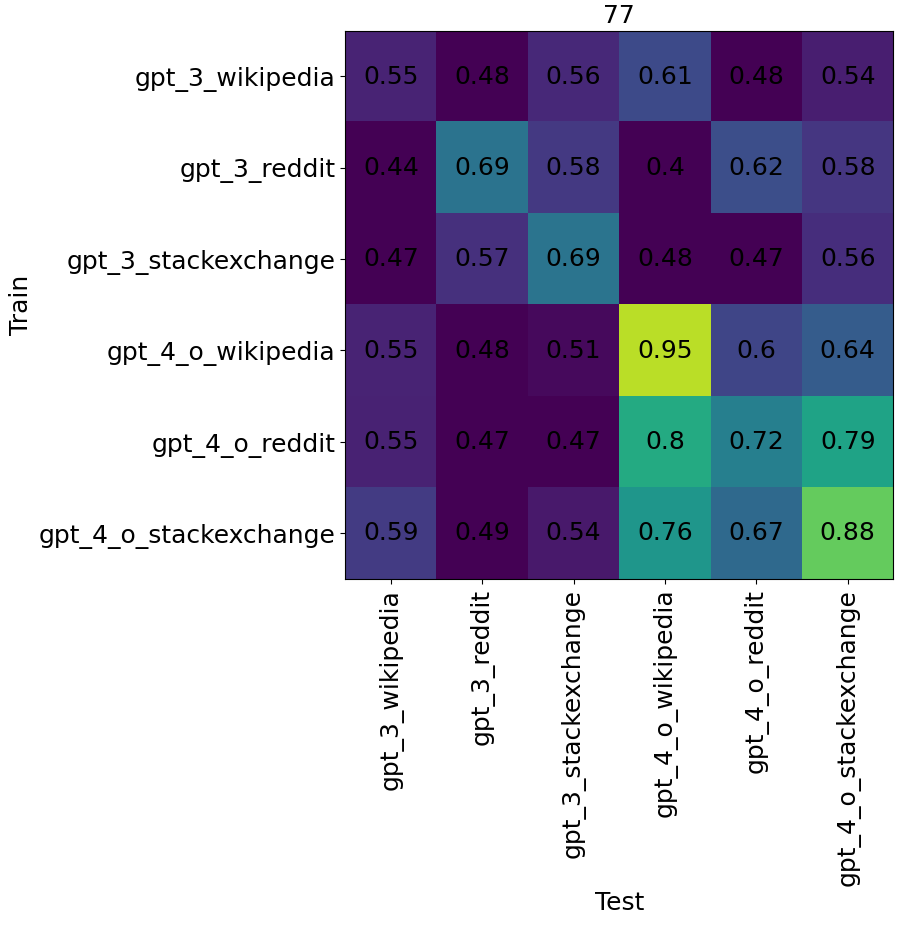}
   \caption{Classification quality on PCA components of RoBERTa embeddings on the GPT-3D dataset. Top left~--- all components are present; top right~--- 10\% of the components with the largest variance are removed; bottom left~--- 50\% of the components with the largest variance are removed; bottom right~--- 90\% of the components with the largest variance are removed.}
   \label{fig:pca_transfer}
\end{figure*}

\begin{figure*}[!t]
   \centering
   \includegraphics[width=.8\textwidth]{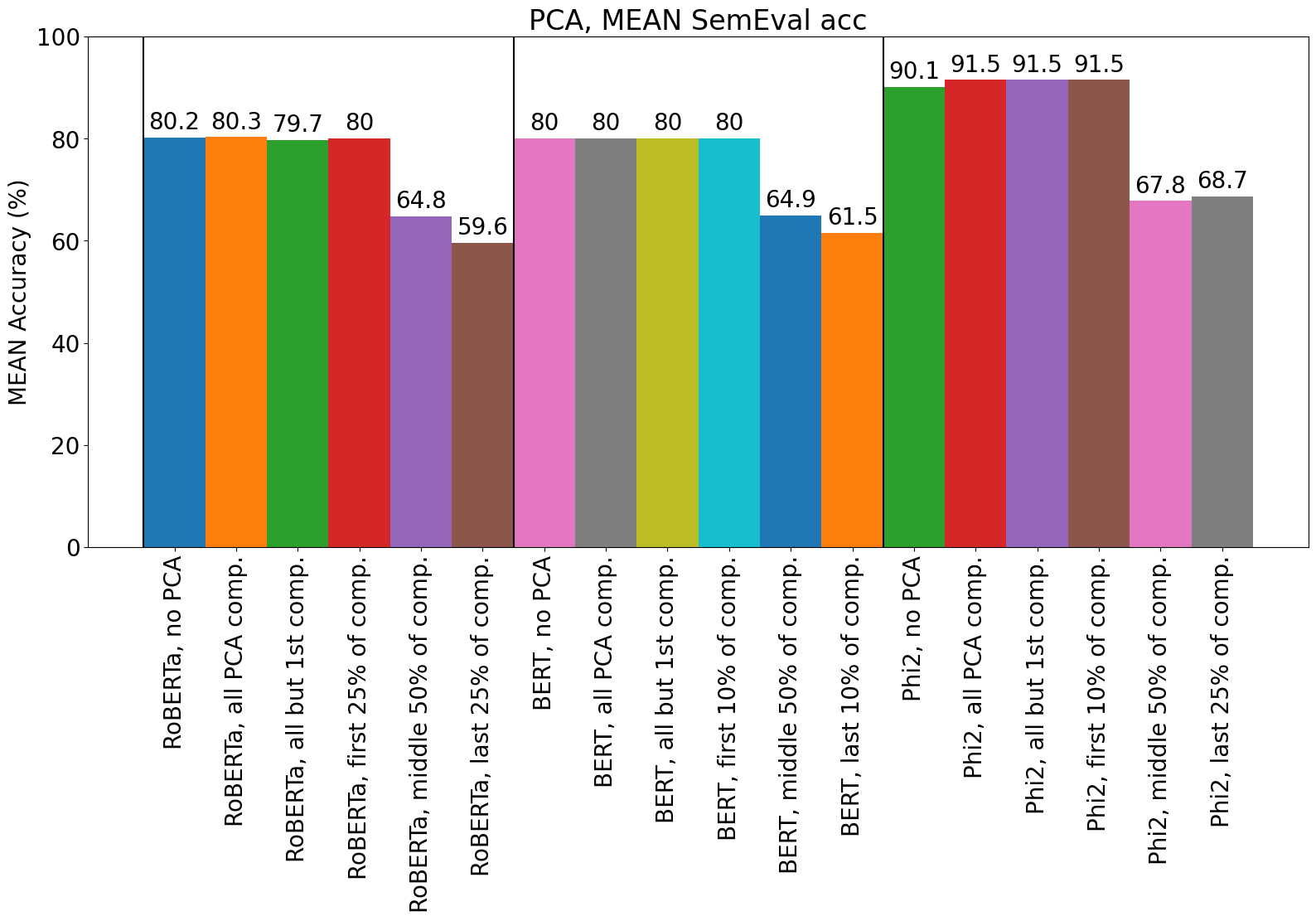}
   \caption{Mean accuracy on the GPT-3D dataset, depending on the number of PCA components left; e.g., ``top 10\% components'' means that we have removed 90\% of the components with the smallest variance.}
   \label{fig:pca_acc_bars}   
\end{figure*}

\section{PCA}
\label{sec:appendix_pca}

We investigated the PCA decomposition of the embedding spaces of RoBERTa, BERT and Phi-2. We tried to remove components with highest and lowest variance to check how it affects the overall accuracy and generalization abilities of the models. The results are shown in Figures~\ref{fig:pca_transfer} and~\ref{fig:pca_acc_bars}.

Figure~\ref{fig:pca_transfer} shows that while we remove PCA components of the RoBERTa embedding space with the largest variance, the transferability between the different domains and models drops significantly. At first, the transferability from GPT-4o to GPT-3.5-davinci goes down to random; next, transferability between different domains of texts generated with GPT-3.5-davinci goes down to random; and finally, transferability between GPT-3.5-davinci and GPT-4o drops down. Interestingly, transferability between different domains of GPT-4o remains significantly higher than random even after removing 90\% of the high-variance components.

Figure~\ref{fig:pca_acc_bars} shows that removing the first PCA component with the highest accuracy does not affect the classification quality much, suggesting that it does not play a distinct role in classification. However, removing 25\% of the components with high variance is damaging for all three models, while removing the components with low or average variance does not hurt the model performance.

Overall, we see that high-variance components in the PCA space generally play some important role in the generalization ability of all three models (RoBERTa, BERT, and Phi-2); however, we have not been able to significantly improve the quality of classification by simply removing low-variance PCA components on any model. 

\section{Datasets license}

We release our dataset under CC BY-SA 4.0 licence agreement. For the information about the licence of M4 (SemEval) subsets, see original paper by \citet{wang-etal-2024-m4}.

\hfill

\end{document}